\documentclass{article}

\usepackage[main, final]{neurips_2026}

\usepackage{amsmath}

\usepackage{xcolor}
\usepackage{booktabs}
\usepackage{multirow}
\usepackage[T1]{fontenc}
\usepackage{FiraMono}

\usepackage{microtype}
\usepackage[normalem]{ulem}

\usepackage[export]{adjustbox}

\usepackage{natbib}
\usepackage{xurl}
\definecolor{mydarkred}{rgb}{0.6,0,0}
\definecolor{mydarkgreen}{rgb}{0,0.6,0}
\usepackage[colorlinks,
linkcolor=mydarkred,
citecolor=mydarkgreen]{hyperref}
\usepackage{graphicx}
\usepackage{subcaption}

\usepackage{booktabs}
\usepackage{multirow}

\usepackage{algorithm}
\usepackage{comment}

\usepackage{math}
\usepackage{enumitem}
\usepackage[capitalize, noabbrev]{cleveref}

\theoremstyle{plain}
\newtheorem{theorem}{Theorem}[section]
\newtheorem{proposition}[theorem]{Proposition}
\newtheorem{lemma}[theorem]{Lemma}
\newtheorem{corollary}[theorem]{Corollary}
\theoremstyle{definition}
\newtheorem{definition}[theorem]{Definition}

\theoremstyle{remark}
\newtheorem{remark}[theorem]{Remark}

\newcommand{\MO}{{\mathrm{MM}}}

\newcommand{\bag}{\mathbf{B}}
\newcommand{\prop}{\alpha}
\newcommand{\propvec}{\valpha}
\newcommand{\classifier}{{\boldsymbol{f}}}
\newcommand{\Loss}{\mathcal{L}} %
\newcommand{\M}{\mathbf{M}} %
\newcommand{\Mhat}{\widehat{\M}} %
\newcommand{\E}{\mathbb{E}} %
\newcommand{\Rademacher}{\mathfrak{R}} %

\newcommand{\indicator}[1]{\mathbb{I}\braces{#1}} %
\newcommand{\order}[1]{^{(#1)}}
\newcommand{\simplex}[1]{\Delta^{#1 - 1}} %

\title{Rethinking Learning from Label Proportions \\ via Moment Matching}

\author{%
  \textbf{Tianhao Ma}$^{1}$, \textbf{Wei Wang}$^{2}$, \textbf{Yivan Zhang}$^{1,2}$, \textbf{Dong-Dong Wu}$^{1,2}$,\\
  \textbf{Takashi Ishida}$^{2,1}$, \textbf{Gang Niu}$^{2}$, \textbf{Masashi Sugiyama}$^{2,1}$\\
  $^{1}$The University of Tokyo, $^{2}$RIKEN AIP
}

\begin{document}

\maketitle

\begin{abstract}
Learning from label proportions (LLP) is a weakly supervised setting where training data are grouped into bags with only bag-level label proportions observed, and the goal is to learn a classifier that predicts labels for individual instances.
Recently, a large body of methods has emerged; however, most of them rely on the assumption that instance--label pairs are sampled i.i.d.~and randomly assembled into bags, an assumption often violated in real-world scenarios.
Our experimental results indicate that these methods perform suboptimally under non-random settings.
In this work, we adopt a more realistic assumption: bags are sampled i.i.d., and instance labels are conditionally independent given the instances.
Under this setting, we show that the label counts follow a Poisson multinomial distribution.
Motivated by this observation, we propose LLP via moment matching~(LLP-MM), a simple yet effective approach which leverages multi-order factorial moments as training objectives, encouraging the classifier predictions to match these moments and thereby more fully exploit the underlying statistical structure of the data.
Extensive experiments on benchmark datasets under various bag construction strategies demonstrate the effectiveness of our approach while maintaining high computational efficiency.
Our code is available at {\hypersetup{urlcolor=black}\urlstyle{same}\url{https://github.com/TianhaoMa5/LLP-MM}}.

\end{abstract}

\section{Introduction}
Recently, the use of large-scale data has pushed machine learning and artificial intelligence to unprecedented levels of performance.
However, in many real-world scenarios the training data are imperfect: labels may be ambiguous, weak, or even incorrect.
These issues have motivated a popular machine learning paradigm, weakly supervised learning~(WSL)~\citep{sugiyama2022machine}, which aims to train robust models from such imperfect supervision. Representative forms of WSL include semi-supervised learning~\citep{li2021semi,lai2022smoothed}, noisy-label learning~\citep{patrini2017making}, partial-label learning~\citep{lv2020progressive,wang2025rethinking}, and unlabeled-unlabeled learning~\citep{lu2021binary,wei2024consistent}.

Learning from label proportions~(LLP) is an important WSL problem that has recently attracted a lot of attention~\citep{quadrianto2009estimating,tsai2020learning,zhang2022learning,ma2025forming,ma2026learning}.
It naturally arises in many real-world applications, such as privacy-sensitive scenarios~\citep{diemert2022lessons,apple2025skan4,google2024privateaggregation}, as well as in image analysis tasks where instance-level annotations are difficult or costly to obtain~\citep{lai2014video,tokunaga2020negative}.
In these scenarios, only label proportions are available for training instead of instance-level annotations, while the goal remains to learn a classifier that predicts labels for individual instances.

In recent years, many LLP methods with strong performance as well as theoretical guarantees have been proposed~\citep{busa-fekete2023easy,li2024optimistic,busa2025nearly}. These methods are often based on the \emph{random-bag assumption} about the generation process of LLP data. First, instance--label pairs are sampled independently from the same joint density. Then, the bag is constructed by grouping instances uniformly at random. Finally, the label proportion is calculated as the fraction of instances from each class. However, in many real-world applications, bags are rarely formed uniformly at random. For example, election polling data is a practical application of LLP where instance-level labels are confidential while the results of regional support rates can be considered as label proportions~\citep{sun2017probabilistic,ma2026learning}. Here, voters within the same region often exhibit correlated voting behaviors, so the resulting groups are affected by geographic and demographic properties rather than random assignment.
Similar structured-bag grouping mechanisms also arise in online advertising and privacy-preserving measurement, where instances are often grouped according to pre-defined categorical attributes or platform-defined grouping keys, rather than being grouped uniformly at random~\citep{diemert2022lessons,brahmbhatt2024llp}. 
\begin{figure}[t]
  \centering
\includegraphics[width=1\linewidth]{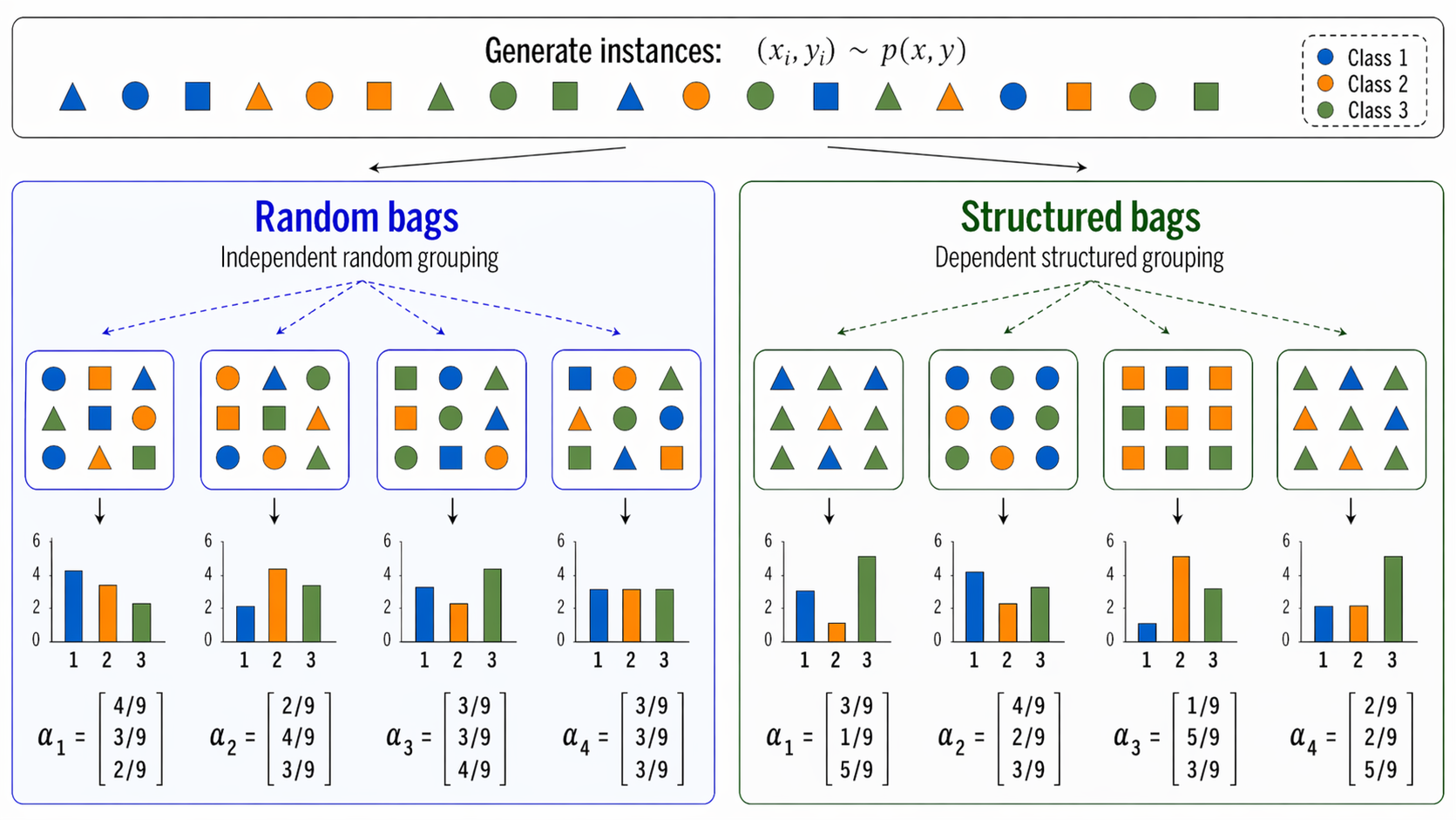}
\caption{An illustration of structured bag construction based on instance features. A common example of a structured-bag assumption is to group instances with similar feature patterns into the same bag; for example, as shown in the figure, instances sharing the same shape feature are assigned to one bag.}\label{fig:bag_con}
\end{figure}

To address this problem, we adopt a more realistic data generation assumption, referred to as the \emph{structured-bag assumption}, as illustrated in Fig.~\ref{fig:bag_con}. In this setting, we allow instance correlations within the bag and assume that the grouping strategy is dependent on the individual examples. Besides, we assume that the instance labels are conditionally independent given the bag. Based on this assumption, we observe that the bag-level label counts correspond to the sum of independent but non-identically distributed categorical random variables, and therefore follow a Poisson multinomial distribution~(PMD)~\citep{daskalakis2015structure}. Consequently, the bag-level label statistics can be considered as following a normalized version of the PMD. This PMD perspective motivates the design of LLP algorithms that exploit the distributional information contained in the observed bag proportions. 

Although it is natural to consider maximum likelihood estimation for LLP data, the likelihood of the PMD is difficult to compute, as its complexity grows exponentially with the number of classes~\citep{lin2023computing}. Therefore, we propose \emph{LLP via Moment Matching~(LLP-MM)}, a simple yet effective approach that trains the classifier by matching classifier-induced multi-order factorial moments to the corresponding empirical factorial moments derived from the observed bag proportions. We also find that the classical proportion matching strategy~\citep{yu2014learning,ardehaly2017co} is a special case of our framework that only uses the first-order empirical moment as the target for the classifier-induced first-order moments. Theoretically, we show that considering higher orders of factorial moments introduces richer statistical information and yields a tighter solution space to facilitate classifier training, which is in accord with our empirical results. Extensive comparative experiments show that moment matching remains effective under low computational cost, achieves state-of-the-art performance across various benchmark datasets, and is robust to different bag generation strategies.

Our contributions can be summarized as follows:
\begin{itemize}[leftmargin=1em, itemsep=1pt, topsep=0pt, parsep=0pt]
  \item We investigate the structured-bag assumption where the grouping strategy is dependent on examples. This assumption is more practical than the existing LLP data generation processes where bags are formed uniformly at random. Based on this assumption, we show that the bag-level label statistics follow the PMD, which offers a fresh perspective about LLP data.  
  \item Motivated by the PMD structure, we propose moment matching, a simple and general framework that trains classifiers by matching classifier-induced moments with empirical multi-order factorial moments derived from label proportions. Our framework includes proportion matching, a widely used LLP strategy, as the special case of the first-order moments.
  \item Extensive experiments demonstrate that moment matching achieves strong performance on different benchmark datasets. Besides, moment matching is robust to different bag generation strategies and attains state-of-the-art results under practical computational budgets.
\end{itemize}

\section{Preliminaries}
\label{sec:preliminaries}
In this section, we introduce the notation and background of multi-class classification and learning from label proportions.

\subsection{Ordinary Multi-class Classification}

Let $\sX$ denote the input space and $\sY \defeq \set{1, \ldots, q}$ denote the label space with $q$ class labels.
We assume that examples are sampled with  probability density $p(x, y)$ over $\sX \times \sY$.
The goal of multi-class classification is to learn a \emph{score-based predictor} $\classifier \in \sF$ in a hypothesis class $\sF$, where $\classifier: \sX \to \simplex{q}$ maps each input to the $(q-1)$-dimensional simplex
\[
\simplex{q}
\defeq
\set*{(p_1, \ldots, p_q) \in \R^q \given p_k \ge 0,\ \sum_{k=1}^q p_k = 1}.
\]
Here, $\classifier_k(x)$ can be interpreted as the predicted probability of class $k$ for input $x$.

Since direct minimization of the empirical $0$-$1$ loss is computationally intractable for many natural hypothesis classes~\citep{ben2003difficulty}, we often instead use a \emph{classification-calibrated surrogate loss} $\ell: \simplex{q} \times \sY \to \R_+$~\citep{bartlett2006convexity} and an empirical approximation of the following classification risk:
\begin{equation}
\label{eq:std-risk}
R(\classifier)
\defeq
\E_{(x, y) \sim p(x, y)} [\ell(\classifier(x), y)].
\end{equation}
Under proper losses such as the cross-entropy loss or the mean squared error loss, the population-risk minimizer recovers the class-posterior distribution. We denote this Bayes-optimal predictor by $\classifier^*$, where
\begin{equation}
\label{eq:bayes-optimal}
\classifier^{*}_k(x)
=
p(y=k \mid x),
\quad
\forall k \in \sY,\ \forall x \in \sX .
\end{equation}
\begin{lemma}
According to \citet{yu2018learning}, if $\ell$ is the cross-entropy loss or the mean squared error loss, then any optimal predictor
\(
\classifier^\star \in \operatorname*{arg\,min}_{\classifier:\sX \to \simplex{q}} R(\classifier)
\)
satisfies $\classifier^\star = \classifier^*$.
\end{lemma}
\subsection{Learning from Label Proportions}

The learning objective in LLP is the same as that in ordinary multi-class classification: we still aim to learn an instance-level predictor $\classifier: \sX \to \simplex{q}$.
The difference is that supervision is provided only at the bag level.
Instead of observing a label for each individual example, the training data consist of groups of examples, referred to as \emph{bags}, together with the class proportions of the labels in each bag~\citep{yu2014learning}.
Following~\citet{busa-fekete2023easy,busa2025nearly}, we assume for simplicity that every bag contains the same number $m$ of examples.
Formally, a \emph{bag} is denoted by $\bag=(x_1,\dots,x_m)\in\sX^m$, where each $x_j$ is an unlabeled example. We assume that each instance-label pair $(x_i, y_i)$ is drawn i.i.d.\ with  probability density $p(x,y)$. These instances are then assembled into bags by an unknown bag-construction mechanism, which may depend on the instance features or other latent grouping factors and is not restricted to uniform random assignment. We assume that, conditional on the bag $\bag=(x_1,\ldots,x_m)$, the labels $y_1,\ldots,y_m$ are mutually independent, with $p(y_j\mid\bag)=p(y_j\mid x_j)$ for every $j\in[m]$. Under this assumption, the conditional probability of the hidden label sequence given the bag can be factorized as 
\begin{equation}
\label{eq:seq-post}
p(\vy\mid\bag)
=
\prod_{j=1}^m p(y_j\mid\bag)
=
\prod_{j=1}^m p(y_j\mid x_j).
\end{equation} 
Instead of observing individual labels, we observe only the \emph{class-proportion vector} $\propvec = (\prop_1, \ldots, \prop_q)$.
For each class $k \in \sY$, $\prop_k \defeq \frac{1}{m} \sum_{j = 1}^m \indicator{y_j = k}$ is the \emph{proportion} of examples in the bag that belong to class $k$. 
Here, $\indicator{\cdot}$ denotes the indicator function, which equals $1$ if the condition is true and $0$ otherwise. This induces an unknown density $p(\bag,\alpha)$ over bags, and we model the observed bags as i.i.d. draws from this bag distribution. 
The posterior probability of $\propvec$ is obtained by marginalizing over all label sequences consistent with $\propvec$:
\begin{equation}
\label{eq:prop-post}
p(\propvec \mid \bag)
= \sum_{\vy \in \sY^{m/\propvec}} p(\vy \mid \bag),
\end{equation}
where $\sY^{m/\propvec}$ denotes the set of all label sequences satisfying the proportion constraints:
\begin{equation}
\sY^{m/\propvec}
=
\set*{
\vy \in \sY^m
\given
\forall k \in \sY,\;
\frac1m \sum_{j = 1}^m \indicator{y_j = k} = \prop_k
}.
\end{equation}
Accordingly, an LLP training dataset $\set{(\bag_i, \propvec_i)}_{i=1}^n$ consists of $n$ i.i.d.~examples of bag-proportion pairs drawn from the induced joint  density $p(\bag, \propvec)$. A common way to learn from such bag-proportion supervision is proportion matching~\citep{yu2014learning,ardehaly2017co}:
given a predictor $\classifier$, one averages its instance-level predictions over each bag
and encourages the resulting predicted class proportion to match the observed proportion
$\alpha$.

\section{Methodology}
\label{sec:method}
In this section, we introduce the proposed moment matching approach. We first establish the connection between LLP and the PMD, and show that the popular proportion matching approach can be considered as the first-order moment matching. Then, we present the second-order and $k$-th-order moment matching. Finally, we present the theoretical analysis for the proposed approach.

\subsection{Rethinking Label Proportions: A PMD Perspective}

We show the connection between LLP and the PMD under certain assumptions about the data generation process. We denote by $\ve^{y_j} \in \set{0, 1}^q$ the one-hot representation of label $y_j$.
Conditioned on $x_j$, this one-hot vector can be equivalently viewed as a draw from a categorical distribution with parameter $p(\cdot \mid x_j)$, namely,
$
\ve^{y_j} \sim \mathrm{Categorical}(p(\cdot \mid x_j)).
$
We further define the bag-level label count vector by
$
\vz \defeq m \propvec = \sum_{j=1}^m \ve^{y_j}.
$
Since $\ve^{y_j}$ are independent categorical random vectors conditioned on $\bag = (x_1, \dots, x_m)$, the count vector $\vz$ follows a PMD~\citep{daskalakis2015structure}.
A straightforward strategy is to directly maximize the likelihood of the PMD in \cref{eq:prop-post}.
However, exact computation of the likelihood is computationally challenging.
In particular, even the state-of-the-art method in computational statistics~\citep{lin2023computing} has a time complexity of $O(m^{q-1} \log m)$.
In typical LLP experimental settings, e.g., $m=64$ and $q=10$~\citep{ma2026learningfrom}, this complexity is already astronomically large and therefore impractical for training.
Instead of maximum likelihood, we leverage moment-based information and train the classifier by matching classifier-induced moments to empirical multi-order factorial moments.
\subsection{Label Proportions Beyond First-Order Moments: A Second-Order Illustration}
The classical algorithm Proportion Matching~\citep{yu2014learning,ardehaly2017co} can be viewed as a special case of our framework that performs first-order moment matching, and therefore fails to exploit higher-order information; see Appendix~\ref{first_order_matching} for details. Under the PMD, we provide an effective solution for exploiting richer information in label proportions through higher-order moments. As an illustration, we consider the following \emph{second-order empirical moment matrix}:
\begin{equation}
\label{eq:second-order}
\Mhat\order{2}_\propvec
=
\frac{1}{m(m - 1)}
\sum_{i \neq j}
\ve^{y_i}(\ve^{y_j})^\top,
\end{equation}
where each entry $\bracks*{\Mhat\order{2}_\propvec}_{c,d}$ represents the normalized number of ordered instance index pairs $(i,j)$ with $i \neq j$, $y_i = c$, and $y_j = d$.
Since the number of instances in class $c$ is $m\prop_c$, this quantity can be expressed entirely in terms of $\propvec$: when $c \neq d$, there are $m\prop_c\, m\prop_d$ such ordered pairs, while when $c=d$, there are $m\prop_c(m\prop_c-1)$, since the same instance cannot be counted twice.
This yields
\begin{equation}
\bracks*{\Mhat\order{2}_\propvec}_{c,d}
=
\frac{m\prop_c\parens*{m\prop_d - \indicator{c=d}}}{m(m - 1)},
\qquad c,d \in \set{1, \dots, q}.
\end{equation}
The detailed derivation is provided in Appendix~\ref{app:second-order-derivation}.
For a bag $\bag$, the classifier-induced second-order moment matrix is given by
\begin{equation}
\Mhat\order{2}_\classifier(\bag)
=
\frac{1}{m(m-1)}
\sum_{i \neq j}
\classifier(x_i) \classifier(x_j)^\top.
\label{eq:second-order-model}
\end{equation}
Accordingly, $\Mhat\order{2}_\propvec$ can serve as the target for the classifier-induced second-order moment. Following the objective of first-order moment matching, we minimize the \emph{second-order moment matching risk}
\begin{equation}
\label{eq:second-order-risk}
\E_{p(\bag, \propvec)}
\bracks*{
\Loss\order{2}
\!\parens*{
\Mhat\order{2}_\propvec,
\Mhat\order{2}_\classifier(\bag)
}
},
\end{equation}
where $\mathcal T_r$ denotes the set of order-$r$ tensors with nonnegative entries summing to one, and $
\Loss\order{2}:
\mathcal T_2\times\mathcal T_2
\to
[0,+\infty]
$ measures the discrepancy between the second-order empirical moment and the second-order classifier-induced moment.

\subsection{Generalized Multi-Order Moment Matching}
After the illustrative second-order example above, we now generalize the construction to arbitrary orders $r \le m$, and describe how such multi-order information can be derived from label proportions and exploited for training.
Following \citet{JohnsonKotzBalakrishnan1997} and \citet{NgTianTang2011}, we introduce the normalized $r$-th order factorial moment tensor
$\Mhat\order{r}_\propvec \in (\R^q)^{\otimes r}$ as
\begin{equation}
\label{eq:rth-factorial-tensor-def}
\Mhat\order{r}_\propvec
\defeq
\frac{1}{(m)_r}
\sum_{(j_1, \dots,j_r) \in [m]_{\neq}^r}
\ve^{y_{j_1}} \otimes \ve^{y_{j_2}} \otimes \cdots \otimes \ve^{y_{j_r}},
\end{equation}
where $(m)_r \defeq m(m-1)\cdots(m-r+1)$ is the falling factorial,
$
[m]_{\neq}^r
\defeq
\set{(j_1, \dots, j_r) \in [m]^r \given j_a \neq j_b\,,\forall a \neq b}
$
denotes the set of ordered $r$-tuples of distinct instance indices, and
$\otimes$ denotes the tensor product, so that
$\ve^{y_{j_1}} \otimes \cdots \otimes \ve^{y_{j_r}}$
is an order-$r$ tensor in $(\R^q)^{\otimes r}$.
Notably, the second-order moment matrix in \cref{eq:second-order} is exactly the special case of \cref{eq:rth-factorial-tensor-def} with $r=2$.
Although the individual labels are not observed, the normalized $r$-th order factorial moment tensor can still be computed entirely from the label proportion vector $\propvec$.
To see this, consider any index tuple $(i_1,\dots,i_r)\in[q]^r$. Define
\(
\bm{k}=(k_1,\dots,k_q)\in\mathbb{N}^q,
k_c \defeq \#\{\ell\in[r]\mid i_\ell=c\}, c=1,\dots,q.
\)
Then \(\bm{k}\) records the multiplicity of each class in \((i_1,\dots,i_r)\), and satisfies
\(
\sum_{c=1}^q k_c = r.
\)
In other words, $k_c$ records how many times class $c$ appears in the tuple $(i_1, \dots, i_r)$.
Then the $(i_1, \dots, i_r)$-th entry of $\Mhat\order{r}_\propvec$ is given by
\begin{equation}
\label{eq:rth-factorial-tensor-alpha-form}
\bracks*{\Mhat\order{r}_\propvec}_{i_1, \dots, i_r}
=
\frac{1}{(m)_r}\prod_{c=1}^q (z_c)_{k_c}
=
\frac{m^r}{(m)_r}\prod_{c=1}^q \prod_{t=0}^{k_c-1}
\parens*{\prop_c-\frac{t}{m}}.
\end{equation}
Given a bag $\bag$, for a classifier $\classifier$, we denote by $p_j \defeq \classifier(x_j)\in\simplex{q}$ the predicted class-probability vector of instance $x_j$.
The classifier-induced normalized $r$-th order factorial moment tensor is then defined as
\begin{equation}
\label{eq:model-induced-rth-tensor}
\Mhat\order{r}_\classifier(\bag)
\defeq
\frac{1}{(m)_r}
\sum_{(j_1, \dots, j_r) \in [m]_{\neq}^r}
p_{j_1}\otimes p_{j_2}\otimes\cdots\otimes p_{j_r}.
\end{equation}
Therefore, $\Mhat\order{r}_\propvec$ computed from the observed label proportions can serve as the target tensor for $\Mhat\order{r}_\classifier(\bag)$, yielding a natural multi-order matching objective.
In our work, we adopt a simple and direct strategy: we use the factorial moment tensor computed from the observed label proportions as the target, and encourage the classifier-induced moment tensor to match it at each order.
Accordingly, the expected multi-order matching risk can be defined as
\begin{equation}
\label{eq:multi-order-risk}
R_\MO\order{s}(\classifier)
\defeq
\E_{p(\bag, \propvec)}
\bracks*{
\sum_{r=1}^{s}
\lambda_r\,
\Loss\order{r}
\!\parens*{
\Mhat\order{r}_\propvec,
\Mhat\order{r}_\classifier(\bag)
}
},
\end{equation}
where $s\in\set{1, \dots, m}$ is the maximum order of moments used for training, $\lambda_r > 0$ is the weight associated with the $r$-th order moment, and $\sum_{r=1}^s \lambda_r = 1$.
It is immediate that the PM risk in \cref{eq:pm-risk} is recovered as the special case $s=1$ of the above objective.
Here, for each $r\in\set{1, \dots, s}$, the loss function
$
\Loss\order{r}:
\mathcal T_r\times\mathcal T_r
\to
[0,+\infty]
$
measures the discrepancy between the target $r$-th order factorial moment tensor $\Mhat\order{r}_\propvec$ and the classifier-induced tensor $\Mhat\order{r}_\classifier(\bag)$.
The empirical counterpart of $R_\MO\order{s}(\classifier)$ on a dataset
$\set{(\bag_i, \propvec_i)}_{i=1}^n$ is given by
\begin{equation}
\label{eq:multi-order-empirical-risk}
\widehat{R}_\MO\order{s}(\classifier)
\defeq
\frac1n \sum_{i=1}^n
\bracks*{
\sum_{r=1}^s
\lambda_r\,
\Loss\order{r}
\!\parens*{
\Mhat\order{r}_{\propvec_i},
\Mhat\order{r}_\classifier(\bag_i)
}
}.
\end{equation}
Since the normalized factorial moment tensor has nonnegative entries that sum
to one, it can be viewed, after vectorization, as a probability distribution
over ordered class tuples \((c_1,\dots,c_r)\in\set{1,\dots,q}^r\). Therefore,
matching two order-\(r\) factorial moment tensors is equivalent to matching two
distributions over such class tuples, and \(\Loss\order{r}\) can be
instantiated by standard distributional discrepancies, such as mean squared
error or cross-entropy.

\subsection{Theoretical Analysis}
For each bag $\bag$, let
$
\Mhat\order{r}_{\classifier^*}(\bag)
$
denote the $r$-th order classifier-induced moment tensor associated with the Bayes-optimal classifier $\classifier^*$.
Let $\sF_\MO\order{s}$ denote the set of minimizers of \cref{eq:multi-order-risk}.
We then have the following proposition:
\begin{proposition}
\label{prop:mo-nesting}
Assume that the hypothesis class \(\sF\) contains the Bayes-optimal classifier
\(\classifier^*\). For each \(r \in \set{1,\dots,s_2}\), assume that for any
\(\mathcal T_r\)-valued random variable \(Z\), its expectation \(\E[Z]\)
is the unique solution of
\(
\min_{v \in \mathcal T_r}
\E\bracks*{\Loss\order{r}(Z,v)}
\). Assume further that \(\lambda_r>0\) for all \(r \in \set{1,\dots,s_2}\).
Then, for any \(1 \le s_1 < s_2 \le m\), we have
\(
\sF_\MO\order{s_2}
\;\subseteq\;
\sF_\MO\order{s_1}.
\)
That is, using more orders of factorial moments yields a tighter minimizer set.
\end{proposition}
\cref{prop:mo-nesting} partially explains the advantage of our method.
By incorporating multi-order factorial moments as supervision, we impose progressively stronger constraints on the feasible minimizers.
As a result, the solution set becomes tighter while still containing the Bayes-optimal classifier $\classifier^*$.
\begin{theorem}[Generalization bound]
\label{thm:mo-generalization}
Assume that the bag coordinates are exchangeable. Assume that, for each $r \in \set{1, \dots, s}$, the loss $\Loss\order{r}(A, \cdot)$ is $L_r$-Lipschitz with respect to the Frobenius norm, uniformly over all admissible target tensors $A$, and satisfies $0 \le \Loss\order{r}(A, B) \le B_r$ for all admissible $A,B$. Then, for any $\delta \in (0, 1)$, with probability at least $1 - \delta$,
\begin{equation}
\sup_{\classifier \in \sF}
\abs{
  R_\MO\order{s}(\classifier)
- \widehat{R}_\MO\order{s}(\classifier)
}
\le
2\sqrt{2m}\,
\parens*{\sum_{r=1}^{s}\lambda_r L_r\, r}
\parens*{\sum_{c=1}^q \Rademacher_n(\sF_c)}
+
\parens*{\sum_{r=1}^{s}\lambda_r B_r}
\sqrt{\frac{\log(2/\delta)}{2n}}.
\end{equation}
\end{theorem}
Under the current bag-i.i.d. assumption, our analysis establishes the optimality of matching the Bayes-optimal classifier-induced moments.
Although this result does not by itself yield a full classifier-level consistency guarantee of the form \(\bm f \to \bm f^*\), it provides a theoretical justification for the proposed moment-matching principle and, together with the empirical results in Section~\ref{sec:experiments}, suggests that the proposed approach is effective in practice.
Establishing such consistency remains challenging for real-world structured bags~\citep{bortsova2018deep,sun2017probabilistic,pacheco2024vertex}, such as those in blood-cell detection, where shared, unobserved physiological factors may induce donor-specific correlations among cells, making the within-bag dependence structure difficult to characterize. We leave this problem for future work.

For computing the loss, a naive approach would explicitly construct the full
order-\(r\) moment tensor, whose number of entries scales as \(q^r\). To avoid
this dependence, we use a state-compressed dynamic programming method that
reduces both time and memory to depend on the number of compressed histogram
states. For example, in the second-order case, the moment tensor is a symmetric
matrix, and therefore one only needs to compute its distinct entries rather than
all \(q^2\) ordered entries. The details are provided in
Appendix~\ref{Computation_anl}. The resulting complexity is summarized below.
\begin{restatable}[Complexity of the state-compressed computation]{proposition}{Complexity}
\label{thm:dp-complexity-final}
Fix a bag of size $m$ and a maximum order $s\le m$. The compressed computation of all moments up to order $s$ requires
$O\!\parens*{m q \binom{q+s}{s}}$
time and
$O\!\parens*{\binom{q+s}{s}}$
memory per bag.
\end{restatable}
Although this complexity still grows combinatorially with the maximum order
\(s\), in practice only a small number of low-order moments is sufficient to
achieve strong performance.
\section{Experiments}
\label{sec:experiments}
\begin{table*}[!t]
    \caption{Best test accuracy under Random Bag and Cluster Bag constructions (mean $\pm$ std over three runs, \%). The best result for each bag size and average is bold. Ranks use average accuracy with standard competition ranking. $\dagger$ indicates severe overfitting caused by negative loss values, with predictions collapsing to a single class.}
    \label{tab:cifar_random_cluster}
    \centering
    \fontsize{7pt}{7pt}\selectfont
    \setlength\tabcolsep{5pt}
    \renewcommand\arraystretch{1.0}
    \begin{tabular}{clcccccc}
    \toprule
    \textbf{Dataset} & \textbf{Method} & \textbf{16} & \textbf{32} & \textbf{64} & \textbf{128} & \textbf{Average} & \textbf{Rank} \\
    \midrule

    \multicolumn{8}{c}{\textbf{Random Bag}} \\
    \midrule

    \multirow{11}{*}{CIFAR-10}
    & PM                   & 93.51 $\pm$ 0.11 & 90.05 $\pm$ 0.22 & 75.23 $\pm$ 0.70 & 60.02 $\pm$ 0.33 & 79.70 & 4 \\
    & DSQ                  & 93.52 $\pm$ 0.15 & 91.39 $\pm$ 0.16 & 79.86 $\pm$ 0.77 & 61.90 $\pm$ 1.35 & 81.67 & 3 \\
    & LLP-PVC              & 91.95 $\pm$ 0.07 & 90.55 $\pm$ 0.08 & 83.79 $\pm$ 0.57 & 64.44 $\pm$ 1.45 & 82.68 & 2 \\
    & LLP-FC               & 75.29 $\pm$ 0.72 & 61.06 $\pm$ 0.87 & 53.89 $\pm$ 0.61 & 60.01 $\pm$ 0.24 & 62.56 & 7 \\
    & ROT                  & 93.21 $\pm$ 0.09 & \textbf{92.04 $\pm$ 0.10} & 80.52 $\pm$ 0.49 & 35.03 $\pm$ 1.25 & 75.20 & 5 \\
    & EasyLLP              & 32.99 $\pm$ 0.80 & 31.98 $\pm$ 2.69 & 26.08 $\pm$ 0.21 & 26.40 $\pm$ 0.98 & 29.36 & 11 \\
    & EasyLLP-ABS          & 57.26 $\pm$ 2.36 & 50.54 $\pm$ 2.77 & 41.86 $\pm$ 3.53 & 33.80 $\pm$ 1.78 & 45.87 & 9 \\
    & GeneralUPM           & 45.43 $\pm$ 1.05 & 36.98 $\pm$ 3.42 & 29.80 $\pm$ 1.21 & 27.20 $\pm$ 0.68 & 34.85 & 10 \\
    & GeneralUPM-ABS       & 61.42 $\pm$ 1.33 & 52.29 $\pm$ 2.76 & 42.43 $\pm$ 3.18 & 33.08 $\pm$ 2.24 & 47.30 & 8 \\
    & FlowLLP              & 92.65 $\pm$ 0.10 & 88.82 $\pm$ 0.54 & 72.65 $\pm$ 0.69 & 40.12 $\pm$ 6.48 & 73.56 & 6 \\
    & \textbf{LLP-MM}      & \textbf{93.97 $\pm$ 0.12} & 91.99 $\pm$ 0.08 & \textbf{84.13 $\pm$ 0.57} & \textbf{66.74 $\pm$ 0.20} & \textbf{84.21} & \textbf{1} \\
    \midrule

    \multirow{11}{*}{CIFAR-100}
    & PM                   & 71.12 $\pm$ 0.26 & 66.62 $\pm$ 0.05 & 54.34 $\pm$ 0.87 & 17.69 $\pm$ 0.57 & 52.44 & 3 \\
    & DSQ                  & 71.01 $\pm$ 0.34 & 65.78 $\pm$ 0.26 & 45.95 $\pm$ 0.86 & \textbf{25.70 $\pm$ 0.80} & 52.11 & 4 \\
    & LLP-PVC              & 63.53 $\pm$ 0.21 & 61.73 $\pm$ 0.47 & 49.78 $\pm$ 2.78 & 12.13 $\pm$ 1.00 & 46.79 & 6 \\
    & LLP-FC               & 71.77 $\pm$ 0.08 & 67.10 $\pm$ 0.18 & 54.84 $\pm$ 0.84 & 17.39 $\pm$ 0.55 & 52.77 & 2 \\
    & ROT                  & 72.02 $\pm$ 0.13 & 64.42 $\pm$ 0.56 & 14.81 $\pm$ 1.46 & 5.26 $\pm$ 0.15 & 39.13 & 7 \\
    & EasyLLP              & 2.43 $\pm$ 0.44 & 1.95 $\pm$ 0.27 & 1.34 $\pm$ 0.20 & 1.34 $\pm$ 0.05 & 1.77 & 11 \\
    & EasyLLP-ABS          & 17.66 $\pm$ 1.73 & 12.39 $\pm$ 1.44 & 9.62 $\pm$ 0.01 & 6.50 $\pm$ 0.41 & 11.55 & 9 \\
    & GeneralUPM           & 12.68 $\pm$ 0.33 & 8.36 $\pm$ 0.30 & 5.65 $\pm$ 1.05 & 4.07 $\pm$ 0.13 & 7.69 & 10 \\
    & GeneralUPM-ABS       & 24.00 $\pm$ 1.57 & 16.51 $\pm$ 1.61 & 11.33 $\pm$ 0.47 & 7.23 $\pm$ 0.40 & 14.77 & 8 \\
    & FlowLLP              & 71.49 $\pm$ 0.33 & 66.76 $\pm$ 0.06 & 51.68 $\pm$ 0.91 & 12.04 $\pm$ 2.05 & 50.49 & 5 \\
    & \textbf{LLP-MM}      & \textbf{72.06 $\pm$ 0.10} & \textbf{67.14 $\pm$ 0.05} & \textbf{57.79 $\pm$ 0.64} & 23.82 $\pm$ 0.54 & \textbf{55.20} & \textbf{1} \\
    \midrule

    \multirow{11}{*}{miniImageNet}
    & PM                   & 64.76 $\pm$ 0.40 & 56.14 $\pm$ 0.35 & 38.31 $\pm$ 1.06 & 10.01 $\pm$ 2.04 & 42.30 & 3 \\
    & DSQ                  & 63.38 $\pm$ 0.13 & 53.73 $\pm$ 0.40 & 31.45 $\pm$ 0.58 & \textbf{14.81 $\pm$ 0.81} & 40.84 & 4 \\
    & LLP-PVC              & 57.71 $\pm$ 0.14 & 54.63 $\pm$ 0.34 & 24.19 $\pm$ 1.57 & 5.21 $\pm$ 0.43 & 35.44 & 6 \\
    & LLP-FC               & 65.75 $\pm$ 0.05 & 58.13 $\pm$ 0.17 & 39.40 $\pm$ 1.75 & 11.05 $\pm$ 1.72 & 43.58 & 2 \\
    & ROT                  & \textbf{67.57 $\pm$ 0.82} & 41.50 $\pm$ 6.10 & 7.21 $\pm$ 0.41 & 2.64 $\pm$ 0.57 & 29.73 & 7 \\
    & EasyLLP              & 1.92 $\pm$ 0.13 & 1.59 $\pm$ 0.15 & 1.51 $\pm$ 0.08 & 1.36 $\pm$ 0.09 & 1.59 & 11 \\
    & EasyLLP-ABS          & 11.88 $\pm$ 0.57 & 8.15 $\pm$ 0.25 & 5.35 $\pm$ 0.41 & 3.57 $\pm$ 0.28 & 7.24 & 9 \\
    & GeneralUPM           & 9.44 $\pm$ 0.11 & 6.08 $\pm$ 0.46 & 4.00 $\pm$ 0.25 & 2.91 $\pm$ 0.33 & 5.61 & 10 \\
    & GeneralUPM-ABS       & 19.61 $\pm$ 0.40 & 11.81 $\pm$ 0.77 & 7.79 $\pm$ 0.23 & 4.29 $\pm$ 0.52 & 10.87 & 8 \\
    & FlowLLP              & 62.83 $\pm$ 0.86 & 52.84 $\pm$ 1.46 & 28.13 $\pm$ 0.94 & 7.20 $\pm$ 0.77 & 37.75 & 5 \\
    & \textbf{LLP-MM}      & 65.94 $\pm$ 0.03 & \textbf{59.79 $\pm$ 0.14} & \textbf{44.76 $\pm$ 0.29} & 8.49 $\pm$ 1.08 & \textbf{44.74} & \textbf{1} \\
    \midrule

    \multicolumn{8}{c}{\textbf{Cluster Bag}} \\
    \midrule

    \multirow{11}{*}{CIFAR-10}
    & PM & 93.31 $\pm$ 0.19 & 91.29 $\pm$ 0.11 & 86.30 $\pm$ 0.45 & 72.94 $\pm$ 1.64 & 85.96 & 2 \\
    & DSQ & 93.62 $\pm$ 0.19 & 91.97 $\pm$ 0.44 & 87.22 $\pm$ 0.53 & 64.65 $\pm$ 2.56 & 84.37 & 4 \\
    & LLP-PVC & 92.30 $\pm$ 0.06 & 90.46 $\pm$ 0.24 & 85.29 $\pm$ 0.69 & 64.09 $\pm$ 2.06 & 83.04 & 6 \\
    & LLP-FC & 56.22 $\pm$ 1.47 & 39.49 $\pm$ 0.69 & 34.29 $\pm$ 0.19 & 72.65 $\pm$ 1.04 & 50.66 & 7 \\
    & ROT & 93.53 $\pm$ 0.07 & 92.05 $\pm$ 0.10 & 87.14 $\pm$ 0.15 & 59.96 $\pm$ 2.21 & 83.17 & 5 \\
    & EasyLLP & 10.90 $\pm$ 1.25\,$\dagger$ & 10.90 $\pm$ 1.26\,$\dagger$ & 10.87 $\pm$ 1.20\,$\dagger$ & 10.78 $\pm$ 1.16\,$\dagger$ & 10.86\,$\dagger$ & 10 \\
    & EasyLLP-ABS & 32.47 $\pm$ 1.31 & 26.93 $\pm$ 1.30 & 23.88 $\pm$ 0.35 & 19.73 $\pm$ 1.65 & 25.76 & 8 \\
    & GeneralUPM & 10.90 $\pm$ 1.25\,$\dagger$ & 10.90 $\pm$ 1.26\,$\dagger$ & 10.87 $\pm$ 1.20\,$\dagger$ & 10.78 $\pm$ 1.16\,$\dagger$ & 10.86\,$\dagger$ & 10 \\
    & GeneralUPM-ABS & 32.47 $\pm$ 1.31 & 26.93 $\pm$ 1.30 & 22.33 $\pm$ 1.84 & 20.25 $\pm$ 1.33 & 25.50 & 9 \\
    & FlowLLP & 93.16 $\pm$ 0.16 & 91.20 $\pm$ 0.17 & 83.96 $\pm$ 0.28 & 69.56 $\pm$ 1.10 & 84.47 & 3 \\
    & \textbf{LLP-MM} & \textbf{94.04 $\pm$ 0.20} & \textbf{92.33 $\pm$ 0.10} & \textbf{89.12 $\pm$ 0.43} & \textbf{75.19 $\pm$ 2.14} & \textbf{87.67} & \textbf{1} \\
    \midrule
    \multirow{11}{*}{CIFAR-100}
    & PM & 71.35 $\pm$ 0.20 & 66.21 $\pm$ 0.68 & 57.78 $\pm$ 0.72 & 46.06 $\pm$ 1.02 & 60.35 & 3 \\
    & DSQ & \textbf{72.33 $\pm$ 0.15} & 67.59 $\pm$ 1.04 & 54.33 $\pm$ 2.31 & 35.59 $\pm$ 1.86 & 57.46 & 5 \\
    & LLP-PVC & 65.03 $\pm$ 0.55 & 61.13 $\pm$ 0.68 & 50.67 $\pm$ 0.69 & 29.77 $\pm$ 0.74 & 51.65 & 7 \\
    & LLP-FC & 71.44 $\pm$ 0.51 & 66.59 $\pm$ 0.50 & 58.44 $\pm$ 1.02 & \textbf{46.19 $\pm$ 0.66} & 60.66 & 2 \\
    & ROT & 72.06 $\pm$ 0.10 & 66.68 $\pm$ 0.28 & 48.45 $\pm$ 0.12 & 29.43 $\pm$ 1.12 & 54.16 & 6 \\
    & EasyLLP & 0.90 $\pm$ 0.02\,$\dagger$ & 0.90 $\pm$ 0.02\,$\dagger$ & 0.90 $\pm$ 0.02\,$\dagger$ & 0.90 $\pm$ 0.02\,$\dagger$ & 0.90\,$\dagger$ & 10 \\
    & EasyLLP-ABS & 7.07 $\pm$ 0.29 & 4.95 $\pm$ 0.41 & 3.22 $\pm$ 0.35 & 2.54 $\pm$ 0.27 & 4.44 & 8 \\
    & GeneralUPM & 0.90 $\pm$ 0.02\,$\dagger$ & 0.90 $\pm$ 0.02\,$\dagger$ & 0.90 $\pm$ 0.02\,$\dagger$ & 0.90 $\pm$ 0.02\,$\dagger$ & 0.90\,$\dagger$ & 10 \\
    & GeneralUPM-ABS & 7.14 $\pm$ 0.36 & 4.74 $\pm$ 0.06 & 3.26 $\pm$ 0.29 & 2.56 $\pm$ 0.25 & 4.42 & 9 \\
    & FlowLLP & 70.41 $\pm$ 0.11 & 65.12 $\pm$ 0.07 & 55.15 $\pm$ 0.39 & 42.62 $\pm$ 0.52 & 58.32 & 4 \\
    & \textbf{LLP-MM} & 72.24 $\pm$ 0.67 & \textbf{67.89 $\pm$ 0.67} & \textbf{58.84 $\pm$ 0.65} & 45.66 $\pm$ 0.50 & \textbf{61.16} & \textbf{1} \\
    \midrule
    \multirow{11}{*}{miniImageNet}
    & PM & 65.21 $\pm$ 0.76 & 56.37 $\pm$ 0.10 & 44.06 $\pm$ 1.78 & 30.18 $\pm$ 0.90 & 48.95 & 3 \\
    & DSQ & 66.03 $\pm$ 0.27 & 55.79 $\pm$ 1.36 & 37.02 $\pm$ 2.75 & 19.56 $\pm$ 0.15 & 44.60 & 4 \\
    & LLP-PVC & 58.77 $\pm$ 0.39 & 52.50 $\pm$ 1.10 & 27.91 $\pm$ 8.11 & 15.31 $\pm$ 1.02 & 38.62 & 7 \\
    & LLP-FC & 65.53 $\pm$ 0.62 & 56.41 $\pm$ 0.61 & 44.26 $\pm$ 1.01 & \textbf{30.30 $\pm$ 1.07} & 49.13 & 2 \\
    & ROT & \textbf{67.29 $\pm$ 0.40} & 58.16 $\pm$ 0.87 & 33.94 $\pm$ 1.89 & 17.83 $\pm$ 1.32 & 44.30 & 5 \\
    & EasyLLP & 1.04 $\pm$ 0.05\,$\dagger$ & 1.01 $\pm$ 0.01\,$\dagger$ & 1.01 $\pm$ 0.01\,$\dagger$ & 1.01 $\pm$ 0.01\,$\dagger$ & 1.01\,$\dagger$ & 10 \\
    & EasyLLP-ABS & 5.31 $\pm$ 0.07 & 3.38 $\pm$ 0.33 & 2.86 $\pm$ 0.12 & 2.27 $\pm$ 0.05 & 3.45 & 8 \\
    & GeneralUPM & 1.04 $\pm$ 0.05\,$\dagger$ & 1.01 $\pm$ 0.01\,$\dagger$ & 1.01 $\pm$ 0.01\,$\dagger$ & 1.01 $\pm$ 0.01\,$\dagger$ & 1.01\,$\dagger$ & 10 \\
    & GeneralUPM-ABS & 5.31 $\pm$ 0.07 & 3.46 $\pm$ 0.26 & 2.78 $\pm$ 0.19 & 2.19 $\pm$ 0.07 & 3.44 & 9 \\
    & FlowLLP & 61.86 $\pm$ 0.88 & 51.56 $\pm$ 0.43 & 37.44 $\pm$ 0.45 & 25.67 $\pm$ 0.53 & 44.13 & 6 \\
    & \textbf{LLP-MM} & 66.99 $\pm$ 0.68 & \textbf{59.65 $\pm$ 0.97} & \textbf{45.91 $\pm$ 1.63} & 28.36 $\pm$ 0.71 & \textbf{50.23} & \textbf{1} \\
    \bottomrule
    \end{tabular}
\end{table*}
\begin{table*}[t]
    \caption{Best test accuracy under $\propvec$-First Bag construction (mean $\pm$ std over three runs, \%). The best result for each bag size and average is bold. Ranks use average accuracy with standard competition ranking.}
    \label{tab:cifar_alphafirst}
    \centering
    \fontsize{7pt}{7pt}\selectfont
    \setlength\tabcolsep{5pt}
    \renewcommand\arraystretch{1.0}
    \begin{tabular}{clcccccc}
    \toprule
    \textbf{Dataset} & \textbf{Method} & \textbf{16} & \textbf{32} & \textbf{64} & \textbf{128} & \textbf{Average} & \textbf{Rank} \\
    \midrule

    \multirow{11}{*}{CIFAR-10}
    & PM                   & 91.00 $\pm$ 0.40 & 70.37 $\pm$ 0.32 & 65.38 $\pm$ 1.14 & 61.35 $\pm$ 0.52 & 72.03 & 4 \\
    & DSQ                  & 90.98 $\pm$ 0.16 & 74.24 $\pm$ 0.57 & 73.78 $\pm$ 1.11 & 68.83 $\pm$ 1.43 & 76.96 & 3 \\
    & LLP-PVC              & 91.85 $\pm$ 0.13 & 89.18 $\pm$ 0.11 & \textbf{81.44 $\pm$ 0.61} & 66.78 $\pm$ 0.54 & 82.31 & 2 \\
    & LLP-FC               & 58.53 $\pm$ 0.36 & 49.40 $\pm$ 0.64 & 49.00 $\pm$ 0.42 & 64.55 $\pm$ 1.15 & 55.37 & 7 \\
    & ROT                  & 92.93 $\pm$ 0.07 & \textbf{90.71 $\pm$ 0.11} & 61.53 $\pm$ 3.70 & 38.24 $\pm$ 1.22 & 70.85 & 5 \\
    & EasyLLP              & 47.72 $\pm$ 1.40 & 37.40 $\pm$ 0.61 & 28.82 $\pm$ 2.44 & 24.09 $\pm$ 1.58 & 34.51 & 10 \\
    & EasyLLP-ABS          & 44.97 $\pm$ 0.77 & 36.84 $\pm$ 0.93 & 38.03 $\pm$ 0.35 & 39.59 $\pm$ 1.07 & 39.86 & 8 \\
    & GeneralUPM           & 41.79 $\pm$ 0.78 & 31.35 $\pm$ 2.18 & 29.76 $\pm$ 0.92 & 28.91 $\pm$ 1.56 & 32.95 & 11 \\
    & GeneralUPM-ABS       & 43.53 $\pm$ 1.03 & 36.56 $\pm$ 0.08 & 37.56 $\pm$ 0.59 & 37.93 $\pm$ 4.25 & 38.89 & 9 \\
    & FlowLLP              & 90.63 $\pm$ 0.08 & 72.68 $\pm$ 1.60 & 60.78 $\pm$ 0.28 & 55.39 $\pm$ 3.44 & 69.87 & 6 \\
    & \textbf{LLP-MM}      & \textbf{93.54 $\pm$ 0.12} & 88.25 $\pm$ 0.58 & 76.83 $\pm$ 0.94 & \textbf{70.65 $\pm$ 0.69} & \textbf{82.32} & \textbf{1} \\
    \midrule

    \multirow{11}{*}{CIFAR-100}
    & PM                   & 70.81 $\pm$ 0.23 & 67.01 $\pm$ 0.14 & 54.53 $\pm$ 0.02 & 5.23 $\pm$ 0.34 & 49.40 & 3 \\
    & DSQ                  & 70.96 $\pm$ 0.45 & 66.07 $\pm$ 0.47 & 46.56 $\pm$ 0.68 & \textbf{13.60 $\pm$ 0.25} & 49.30 & 4 \\
    & LLP-PVC              & 64.54 $\pm$ 0.34 & 61.43 $\pm$ 0.29 & 49.81 $\pm$ 1.90 & 4.41 $\pm$ 0.53 & 45.05 & 6 \\
    & LLP-FC               & 71.41 $\pm$ 0.13 & 66.93 $\pm$ 0.16 & 57.69 $\pm$ 0.43 & 4.75 $\pm$ 0.10 & 50.20 & 2 \\
    & ROT                  & \textbf{72.14 $\pm$ 0.23} & 63.96 $\pm$ 0.40 & 16.58 $\pm$ 0.66 & 2.76 $\pm$ 0.04 & 38.86 & 7 \\
    & EasyLLP              & 1.95 $\pm$ 0.25 & 2.43 $\pm$ 0.63 & 2.13 $\pm$ 0.43 & 1.73 $\pm$ 0.16 & 2.06 & 11 \\
    & EasyLLP-ABS          & 17.51 $\pm$ 0.83 & 12.50 $\pm$ 0.36 & 9.39 $\pm$ 0.32 & 4.32 $\pm$ 0.66 & 10.93 & 9 \\
    & GeneralUPM           & 12.27 $\pm$ 0.43 & 8.87 $\pm$ 0.76 & 5.77 $\pm$ 0.18 & 2.63 $\pm$ 0.30 & 7.38 & 10 \\
    & GeneralUPM-ABS       & 23.98 $\pm$ 0.88 & 16.93 $\pm$ 0.54 & 11.20 $\pm$ 0.04 & 6.36 $\pm$ 0.37 & 14.61 & 8 \\
    & FlowLLP              & 70.90 $\pm$ 0.24 & 65.85 $\pm$ 0.46 & 50.37 $\pm$ 1.17 & 2.79 $\pm$ 0.53 & 47.48 & 5 \\
    & \textbf{LLP-MM}      & 72.01 $\pm$ 0.29 & \textbf{67.36 $\pm$ 0.25} & \textbf{58.02 $\pm$ 0.35} & 7.18 $\pm$ 0.20 & \textbf{51.14} & \textbf{1} \\
    \midrule

    \multirow{11}{*}{miniImageNet}
    & PM                   & 63.76 $\pm$ 0.06 & 55.47 $\pm$ 0.29 & 36.44 $\pm$ 0.80 & 3.62 $\pm$ 0.09 & 39.82 & 4 \\
    & DSQ                  & 63.14 $\pm$ 0.33 & 55.09 $\pm$ 0.60 & 33.03 $\pm$ 0.57 & \textbf{8.72 $\pm$ 0.69} & 40.00 & 3 \\
    & LLP-PVC              & 57.47 $\pm$ 0.37 & 54.74 $\pm$ 0.49 & 31.01 $\pm$ 1.47 & 2.89 $\pm$ 0.26 & 36.53 & 6 \\
    & LLP-FC               & 64.23 $\pm$ 0.32 & 54.71 $\pm$ 0.48 & 38.82 $\pm$ 0.38 & 3.80 $\pm$ 0.27 & 40.39 & 2 \\
    & ROT                  & 65.87 $\pm$ 0.48 & 54.29 $\pm$ 1.00 & 11.19 $\pm$ 0.20 & 2.33 $\pm$ 0.28 & 33.42 & 7 \\
    & EasyLLP              & 1.62 $\pm$ 0.21 & 1.61 $\pm$ 0.08 & 1.71 $\pm$ 0.13 & 1.48 $\pm$ 0.08 & 1.60 & 11 \\
    & EasyLLP-ABS          & 11.47 $\pm$ 0.20 & 8.56 $\pm$ 0.18 & 5.49 $\pm$ 0.66 & 2.83 $\pm$ 1.04 & 7.09 & 9 \\
    & GeneralUPM           & 9.85 $\pm$ 0.13 & 6.89 $\pm$ 0.54 & 4.30 $\pm$ 0.21 & 2.41 $\pm$ 0.19 & 5.86 & 10 \\
    & GeneralUPM-ABS       & 20.73 $\pm$ 1.05 & 13.60 $\pm$ 0.78 & 7.70 $\pm$ 0.09 & 4.03 $\pm$ 0.30 & 11.51 & 8 \\
    & FlowLLP              & 62.76 $\pm$ 0.21 & 52.79 $\pm$ 1.01 & 30.89 $\pm$ 1.59 & 2.67 $\pm$ 0.33 & 37.28 & 5 \\
    & \textbf{LLP-MM}      & \textbf{66.37 $\pm$ 0.16} & \textbf{59.96 $\pm$ 0.19} & \textbf{44.91 $\pm$ 0.45} & 3.43 $\pm$ 0.09 & \textbf{43.67} & \textbf{1} \\
    \bottomrule
    \end{tabular}
\end{table*}
\subsection{Experimental Setup}
\subsubsection{Datasets and Bag Construction}
Following prior work~\citep{busa-fekete2023easy,busa2025nearly}, we construct synthetic LLP datasets on three standard benchmarks: CIFAR-10, CIFAR-100~\citep{krizhevsky2009learning}, and miniImageNet~\citep{vinyals2016matching}.
Following~\citet{ma2026learningfrom}, we evaluate three bag generation strategies, namely random bags, cluster bags, and $\propvec$-first bags, with bag sizes $m \in \set{16,32,64,128}$.
Existing LLP benchmarks are mainly designed for binary classification~\citep{brahmbhatt2024llp,franco2023dependence}, and are therefore not well aligned with our multi-class classification setting.
The detailed construction procedures are provided in Appendix~\ref{app:bag-construction}.
\subsubsection{Baselines and Implementation Details}
We compare LLP-MM with representative LLP baselines, including PM~\citep{yu2014learning,ardehaly2017co}, DSQ~\citep{li2024optimistic}, LLP-PVC~\citep{ma2026learningfrom}, LLP-FC~\citep{zhang2022learning}, ROT~\citep{dulac2019deep}, EasyLLP~\citep{busa-fekete2023easy}, GeneralUPM~\citep{applebaum2025optimal}, EasyLLP-ABS, GeneralUPM-ABS, and FlowLLP~\citep{rakotomamonjy2026particle}. EasyLLP-ABS and GeneralUPM-ABS take the absolute value of the original loss to reduce the effect of negative loss values.
For the main benchmark comparisons, LLP-MM uses \(s=8\) for
CIFAR-10, and \(s=3\) for CIFAR-100 and miniImageNet. For each order \(r\), we
use the cross-entropy loss as \(\Loss\order{r}\), and set the weight
\(\lambda_r\) to \(1/s\).
Further implementation details are provided in Appendix~\ref{app:implementation-details}.
We also evaluate on patient-grouped blood cell data in Section~\ref{sec:semi-real-world}.

\subsection{Comparison with Baselines}

LLP-MM ranks first in average accuracy across all nine dataset--bag-construction settings in Tables~\ref{tab:cifar_random_cluster} and~\ref{tab:cifar_alphafirst}. DSQ and LLP-FC perform best in some larger-bag settings. EasyLLP and GeneralUPM can suffer from negative empirical losses during optimization. LLP-PVC is competitive on CIFAR-10 but has lower average accuracy on CIFAR-100 and miniImageNet. This gap may partly reflect its one-vs-rest decomposition, which creates increasingly imbalanced binary tasks as the number of classes grows and does not directly model the joint multiclass PMD. FFT-based evaluation of very small Poisson-binomial probabilities may also affect likelihood optimization through roundoff and cancellation errors~\citep{peres2020exactly}.

\subsection{Semi-Real-World Evaluation}
\label{sec:semi-real-world}
\label{app:semi-real-world}

Inspired by the blood-cell LLP setting of \citet{pacheco2024vertex}, whose real data are not publicly available, we construct a semi-real-world LLP benchmark on the public 13-class KU-Optofil PBC dataset~\citep{yarikan2026large}. We retain its official patient-separated split, merge training and validation cells, and form bags by patient ID, with synthetic subbags for missing IDs. Bag proportions are computed from cell labels. Table~\ref{tab:ku-optofil-semi-real} reports three-run results, selecting each run's checkpoint by the highest test Macro-F1. Details are in Appendix~\ref{app:ku-implementation-details}.

\begin{table}[!t]
  \centering
  \setlength{\belowcaptionskip}{3pt}
  \caption{Test performance on KU-Optofil PBC (mean $\pm$ std over three runs, \%). The best result in each column is bold.}
  \label{tab:ku-optofil-semi-real}
  \footnotesize
  \setlength{\tabcolsep}{4pt}
  \renewcommand{\arraystretch}{0.95}
  \begin{tabular*}{0.96\linewidth}{@{\extracolsep{\fill}}lcccc@{}}
    \toprule
    \textbf{Method} & \textbf{Accuracy} $\uparrow$ & \textbf{Macro-F1} $\uparrow$ & \textbf{Balanced Acc.} $\uparrow$ & \textbf{Weighted-F1} $\uparrow$ \\
    \midrule
    PM             & $77.35\pm13.10$ & $29.11\pm10.74$ & $32.70\pm10.01$ & $74.46\pm15.73$ \\
    DSQ            & $92.39\pm1.33$ & $56.89\pm7.65$ & $60.08\pm8.30$ & $92.70\pm1.57$ \\
    LLP-PVC        & $92.82\pm1.94$ & $58.42\pm4.31$ & $61.63\pm4.27$ & $91.58\pm2.33$ \\
    LLP-FC         & $75.11\pm9.57$ & $26.26\pm4.28$ & $27.12\pm4.46$ & $72.46\pm11.60$ \\
    ROT            & $\mathbf{94.05\pm0.65}$ & $56.11\pm7.56$ & $56.82\pm7.74$ & $93.35\pm1.17$ \\
    EasyLLP        & $37.98\pm8.32$ & $11.51\pm1.77$ & $14.50\pm0.85$ & $39.28\pm3.11$ \\
    EasyLLP-ABS    & $38.33\pm8.83$ & $11.02\pm1.59$ & $13.85\pm1.81$ & $40.48\pm5.11$ \\
    GeneralUPM     & $66.33\pm0.22$ & $12.48\pm1.28$ & $14.16\pm1.32$ & $55.19\pm2.51$ \\
    GeneralUPM-ABS & $66.33\pm0.22$ & $12.48\pm1.28$ & $14.16\pm1.32$ & $55.19\pm2.51$ \\
    FlowLLP        & $61.97\pm4.34$ & $25.67\pm2.34$ & $31.11\pm1.90$ & $55.75\pm4.85$ \\
    \midrule
    LLP-MM ($s=3$) & $91.88\pm1.20$ & $51.79\pm5.80$ & $53.23\pm7.41$ & $91.88\pm0.48$ \\
    LLP-MM ($s=5$) & $93.40\pm0.84$ & $66.40\pm2.22$ & $69.81\pm0.91$ & $94.20\pm0.55$ \\
    \textbf{LLP-MM ($s=8$)} & $93.93\pm0.58$ & $\mathbf{69.75\pm1.47}$ & $\mathbf{70.73\pm2.74}$ & $\mathbf{94.57\pm0.22}$ \\
    \bottomrule
  \end{tabular*}
\end{table}

LLP-MM ($s=8$) leads in Macro-F1, Balanced Accuracy, and Weighted-F1, exceeding LLP-PVC by 11.33 Macro-F1 points; ROT has the highest Accuracy. Increasing $s$ from 3 to 5 to 8 raises Macro-F1 from $51.79\%$ to $66.40\%$ to $69.75\%$.

\begin{figure}[!t]
    \centering
    \includegraphics[width=\linewidth]{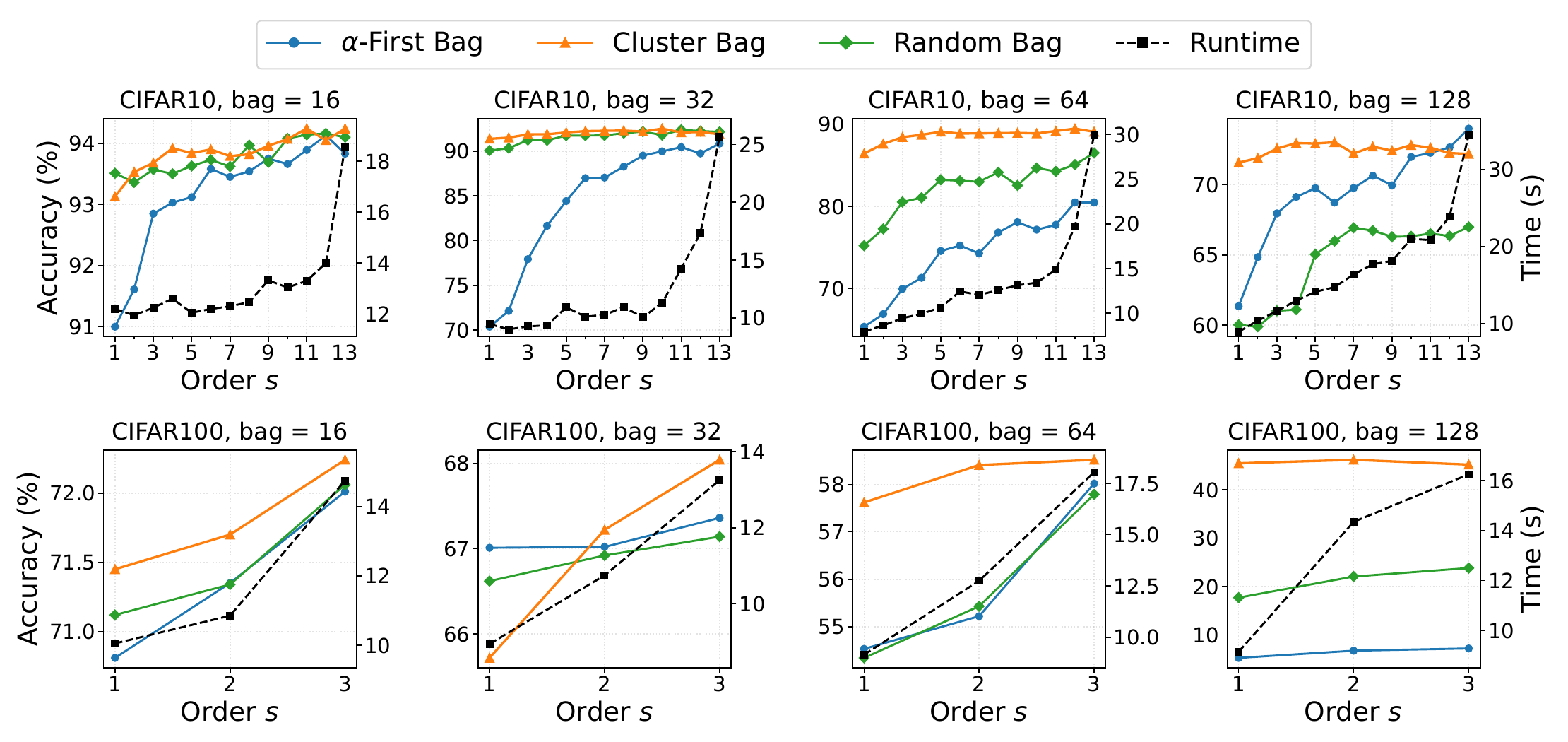}
    \caption{Effect of the order parameter on prediction performance and runtime per epoch across different bag construction strategies and bag sizes. Runtime was measured on a single NVIDIA A100-SXM4-40GB GPU with an AMD EPYC 7742 64-Core Processor.}
    \label{fig:order-effect-runtime-allbags}
\end{figure}

\subsection{Effect of the Order Parameter}

Figure~\ref{fig:order-effect-runtime-allbags} shows the accuracy--runtime trade-off as $s$ increases. Higher orders tighten the minimizer set toward $\bm f^*$ (Proposition~\ref{prop:mo-nesting}) but increase computation. We report CIFAR-100 only up to $s=3$, since fourth-order terms are prohibitively expensive with 100 classes. As the order increases from low values, prediction performance generally improves, suggesting that higher-order moments provide useful supervision beyond mean class proportions. At higher orders, performance tends to stabilize, and further increases bring only limited gains. Since computation continues to grow with the order, moderate orders offer a practical balance between predictive performance and computational cost.

\section{Conclusion}
We proposed LLP-MM, a moment-matching framework for learning from label proportions. 
By connecting proportions to the Poisson multinomial distribution, LLP-MM exploits multi-order factorial moments derived from observed label proportions and uses them as training supervision. 
Theoretical analysis shows that higher-order moments impose stronger constraints on the solution space, and experiments demonstrate that LLP-MM achieves robust performance across various settings while maintaining practical computational efficiency.
A limitation is the scarcity of standardized real-world LLP datasets; beyond the semi-real-world study in Section~\ref{sec:semi-real-world}, future work will develop broader realistic benchmarks. We will also design multiclass LLP benchmarks and evaluation metrics to compare different methods.

\begin{ack}
Part of the experiments was conducted using the Miyabi and Wisteria supercomputing systems at the Joint Center for Advanced High Performance Computing (JCAHPC) and the Information Technology Center, The University of Tokyo, respectively.

YZ was supported by JSPS KAKENHI Grant Number JP26K21299. DDW was supported by the RIKEN Junior Research Associate Program. TI was supported by KAKENHI Grant Number 22K17946. MS was supported by JST ASPIRE Grant Number JPMJAP25B1.

\end{ack}

\bibliography{references}

\begin{thebibliography}{46}
\providecommand{\natexlab}[1]{#1}
\providecommand{\url}[1]{\texttt{#1}}
\expandafter\ifx\csname urlstyle\endcsname\relax
  \providecommand{\doi}[1]{doi: #1}\else
  \providecommand{\doi}{doi: \begingroup \urlstyle{rm}\Url}\fi

\bibitem[{Apple}(2025)]{apple2025skan4}
{Apple}.
\newblock {SKAdNetwork} 4 release notes.
\newblock
  \url{https://developer.apple.com/documentation/storekit/skadnetwork/skadnetwork_release_notes/skadnetwork_4_release_notes},
  2025.
\newblock Apple Developer Documentation.

\bibitem[Applebaum et~al.(2025)Applebaum, Dick, Gentile, Kaplan, and
  Koren]{applebaum2025optimal}
Lorne Applebaum, Travis Dick, Claudio Gentile, Haim Kaplan, and Tomer Koren.
\newblock Optimal learning from label proportions with general loss functions.
\newblock \emph{arXiv preprint arXiv:2509.15145}, 2025.

\bibitem[Ardehaly and Culotta(2017)]{ardehaly2017co}
Ehsan~Mohammady Ardehaly and Aron Culotta.
\newblock Co-training for demographic classification using deep learning from
  label proportions.
\newblock In \emph{IEEE International Conference on Data Mining Workshops},
  pages 1017--1024, 2017.

\bibitem[Bartlett et~al.(2006)Bartlett, Jordan, and
  McAuliffe]{bartlett2006convexity}
Peter~L Bartlett, Michael~I Jordan, and Jon~D McAuliffe.
\newblock Convexity, classification, and risk bounds.
\newblock \emph{Journal of the American Statistical Association}, 101\penalty0
  (473):\penalty0 138--156, 2006.
\newblock URL \url{https://doi.org/10.1198/016214505000000907}.

\bibitem[Ben-David et~al.(2003)Ben-David, Eiron, and Long]{ben2003difficulty}
Shai Ben-David, Nadav Eiron, and Philip~M Long.
\newblock On the difficulty of approximately maximizing agreements.
\newblock \emph{Journal of Computer and System Sciences}, 66\penalty0
  (3):\penalty0 496--514, 2003.
\newblock URL \url{https://doi.org/10.1016/S0022-0000(03)00038-2}.

\bibitem[Bortsova et~al.(2018)Bortsova, Dubost, {\O}rting, Katramados, Hogeweg,
  Thomsen, Wille, and de~Bruijne]{bortsova2018deep}
Gerda Bortsova, Florian Dubost, Silas {\O}rting, Ioannis Katramados, Laurens
  Hogeweg, Laura Thomsen, Mathilde Wille, and Marleen de~Bruijne.
\newblock Deep learning from label proportions for emphysema quantification.
\newblock In \emph{Proceedings of the 21st International Conference on Medical
  Image Computing and Computer Assisted Intervention}, pages 768--776.
  Springer, 2018.

\bibitem[Brahmbhatt et~al.(2024)Brahmbhatt, Pokala, Saket, and
  Raghuveer]{brahmbhatt2024llp}
Anand Brahmbhatt, Mohith Pokala, Rishi Saket, and Aravindan Raghuveer.
\newblock {{LLP-Bench}}: A large scale tabular benchmark for learning from
  label proportions.
\newblock In \emph{Proceedings of the 33rd ACM International Conference on
  Information and Knowledge Management}, pages 4374--4381, 2024.

\bibitem[Busa-Fekete et~al.(2023)Busa-Fekete, Choi, Dick, Gentile, and
  Medina]{busa-fekete2023easy}
R{\'{o}}bert Busa-Fekete, Heejin Choi, Travis Dick, Claudio Gentile, and
  Andr{\'{e}}s~Mu{\~{n}}oz Medina.
\newblock Easy learning from label proportions.
\newblock In \emph{Advances in Neural Information Processing Systems}, pages
  14957--14968, 2023.

\bibitem[Busa-Fekete et~al.(2025)Busa-Fekete, Dick, Gentile, Kaplan, Koren, and
  Stemmer]{busa2025nearly}
Robert Busa-Fekete, Travis Dick, Claudio Gentile, Haim Kaplan, Tomer Koren, and
  Uri Stemmer.
\newblock Nearly optimal sample complexity for learning with label proportions.
\newblock \emph{arXiv preprint arXiv:2505.05355}, 2025.

\bibitem[Daskalakis et~al.(2015)Daskalakis, Kamath, and
  Tzamos]{daskalakis2015structure}
Constantinos Daskalakis, Gautam Kamath, and Christos Tzamos.
\newblock On the structure, covering, and learning of {Poisson} multinomial
  distributions.
\newblock In \emph{2015 IEEE 56th annual symposium on foundations of computer
  science}, pages 1203--1217, 2015.

\bibitem[Diemert et~al.(2022)Diemert, Fabre, Gilotte, Jia, Leparmentier, Mary,
  Qu, Tanielian, and Yang]{diemert2022lessons}
Eustache Diemert, Romain Fabre, Alexandre Gilotte, Fei Jia, Basile
  Leparmentier, J{\'e}r{\'e}mie Mary, Zhonghua Qu, Ugo Tanielian, and Hui Yang.
\newblock Lessons from the {AdKDD}’21 privacy-preserving {ML} challenge.
\newblock In \emph{Proceedings of the ACM Web Conference 2022}, pages
  2026--2035, 2022.

\bibitem[Dulac-Arnold et~al.(2019)Dulac-Arnold, Zeghidour, Cuturi, Beyer, and
  Vert]{dulac2019deep}
Gabriel Dulac-Arnold, Neil Zeghidour, Marco Cuturi, Lucas Beyer, and
  Jean-Philippe Vert.
\newblock Deep multi-class learning from label proportions.
\newblock \emph{arXiv preprint arXiv:1905.12909}, 2019.

\bibitem[Franco et~al.(2023)Franco, Crovella, and
  Comarela]{franco2023dependence}
Gabriel Franco, Mark Crovella, and Giovanni Comarela.
\newblock Dependence and model selection in {LLP}: the problem of variants.
\newblock In \emph{Proceedings of the 29th ACM SIGKDD Conference on Knowledge
  Discovery and Data Mining}, pages 470--481, 2023.

\bibitem[{Google}(2024)]{google2024privateaggregation}
{Google}.
\newblock Private aggregation {API} of {Chrome} privacy sandbox.
\newblock
  \url{https://developer.chrome.com/docs/privacy-sandbox/aggregation-service/},
  2024.
\newblock Chrome Privacy Sandbox documentation.

\bibitem[Johnson et~al.(1997)Johnson, Kotz, and
  Balakrishnan]{JohnsonKotzBalakrishnan1997}
Norman~L. Johnson, Samuel Kotz, and Narayanaswamy Balakrishnan.
\newblock \emph{Discrete Multivariate Distributions}.
\newblock John Wiley \& Sons, 1997.
\newblock ISBN 9780471128441.

\bibitem[Krizhevsky and Hinton(2009)]{krizhevsky2009learning}
Alex Krizhevsky and Geoffrey~E. Hinton.
\newblock Learning multiple layers of features from tiny images.
\newblock Technical report, University of Toronto, 2009.

\bibitem[Lai et~al.(2014)Lai, Yu, Chen, and Chang]{lai2014video}
Kuan-Ting Lai, Felix~X Yu, Ming-Syan Chen, and Shih-Fu Chang.
\newblock Video event detection by inferring temporal instance labels.
\newblock In \emph{Proceedings of the IEEE Conference on Computer Vision and
  Pattern Recognition}, pages 2243--2250, 2014.

\bibitem[Lai et~al.(2022)Lai, Wang, Gunawan, Cheung, and
  Chuah]{lai2022smoothed}
Zhengfeng Lai, Chao Wang, Henrry Gunawan, Sen-Ching~S. Cheung, and Chen-Nee
  Chuah.
\newblock Smoothed adaptive weighting for imbalanced semi-supervised learning:
  Improve reliability against unknown distribution data.
\newblock In \emph{International Conference on Machine Learning}, pages
  11828--11843, 2022.

\bibitem[Li et~al.(2021)Li, Li, and Ouyang]{li2021semi}
Changchun Li, Ximing Li, and Jihong Ouyang.
\newblock Semi-supervised text classification with balanced deep representation
  distributions.
\newblock In \emph{Association for Computational Linguistics and the
  International Joint Conference on Natural Language Processing}, pages
  5044--5053, 2021.

\bibitem[Li et~al.(2024)Li, Chen, Javanmard, and Mirrokni]{li2024optimistic}
Gene Li, Lin Chen, Adel Javanmard, and Vahab Mirrokni.
\newblock Optimistic rates for learning from label proportions.
\newblock In \emph{Annual Conference on Learning Theory}, pages 3437--3474,
  2024.

\bibitem[Lin et~al.(2023)Lin, Wang, and Hong]{lin2023computing}
Zhengzhi Lin, Yueyao Wang, and Yili Hong.
\newblock The computing of the poisson multinomial distribution and
  applications in ecological inference and machine learning.
\newblock \emph{Computational Statistics}, 38\penalty0 (4):\penalty0
  1851--1877, 2023.

\bibitem[Liu et~al.(2019)Liu, Wang, Qi, Tian, and Shi]{Liu2019Learning}
Jiabin Liu, Bo~Wang, Zhiquan Qi, Yingjie Tian, and Yong Shi.
\newblock Learning from label proportions with generative adversarial networks.
\newblock In \emph{Advances in Neural Information Processing Systems}, pages
  7167--7177, 2019.

\bibitem[Lu et~al.(2021)Lu, Lei, Niu, Sato, and Sugiyama]{lu2021binary}
Nan Lu, Shida Lei, Gang Niu, Issei Sato, and Masashi Sugiyama.
\newblock Binary classification from multiple unlabeled datasets via surrogate
  set classification.
\newblock In \emph{International Conference on Machine Learning}, pages
  7134--7144, 2021.

\bibitem[Lv et~al.(2020)Lv, Xu, Feng, Niu, Geng, and
  Sugiyama]{lv2020progressive}
Jiaqi Lv, Miao Xu, Lei Feng, Gang Niu, Xin Geng, and Masashi Sugiyama.
\newblock Progressive identification of true labels for partial-label learning.
\newblock In \emph{International Conference on Machine Learning}, pages
  6500--6510, 2020.

\bibitem[Ma et~al.(2025)Ma, Chen, Hu, Zhu, and Li]{ma2025forming}
Tianhao Ma, Han Chen, Juncheng Hu, Yungang Zhu, and Ximing Li.
\newblock Forming auxiliary high-confident instance-level loss to promote
  learning from label proportions.
\newblock In \emph{Proceedings of the Computer Vision and Pattern Recognition
  Conference}, pages 20592--20601, 2025.

\bibitem[Ma et~al.(2026{\natexlab{a}})Ma, Li, Li, and Guan]{ma2026learning}
Tianhao Ma, Ximing Li, Changchun Li, and Renchu Guan.
\newblock Learning from label proportions with dual-proportion constraints.
\newblock \emph{arXiv preprint arXiv:2603.21153}, 2026{\natexlab{a}}.

\bibitem[Ma et~al.(2026{\natexlab{b}})Ma, Wang, Li, Niu, and
  Sugiyama]{ma2026learningfrom}
Tianhao Ma, Wei Wang, Ximing Li, Gang Niu, and Masashi Sugiyama.
\newblock Learning from label proportions via proportional value
  classification.
\newblock In \emph{The Fourteenth International Conference on Learning
  Representations}, 2026{\natexlab{b}}.

\bibitem[Maurer(2016)]{maurer2016vector}
Andreas Maurer.
\newblock A vector-contraction inequality for {Rademacher} complexities.
\newblock In \emph{Proceedings of the 27th International Conference on
  Algorithmic Learning Theory}, pages 3--17, 2016.

\bibitem[Mohri et~al.(2018)Mohri, Rostamizadeh, and
  Talwalkar]{mohri2018foundations}
Mehryar Mohri, Afshin Rostamizadeh, and Ameet Talwalkar.
\newblock \emph{Foundations of machine learning}.
\newblock MIT press, 2018.

\bibitem[Ng et~al.(2011)Ng, Tian, and Tang]{NgTianTang2011}
Kai~Wang Ng, Guo-Liang Tian, and Man-Lai Tang.
\newblock \emph{Dirichlet and Related Distributions: Theory, Methods and
  Applications}.
\newblock John Wiley \& Sons, 2011.
\newblock ISBN 9780470688199.

\bibitem[Pacheco et~al.(2024)Pacheco, Yellin, Vidal, and
  Haeffele]{pacheco2024vertex}
Carolina Pacheco, Florence Yellin, Ren{\'e} Vidal, and Benjamin Haeffele.
\newblock Vertex proportion loss for multi-class cell detection from label
  proportions.
\newblock In \emph{Medical Image Computing and Computer Assisted Intervention},
  pages 366--376. Springer Nature Switzerland, 2024.
\newblock URL \url{https://papers.miccai.org/miccai-2024/829-Paper3997.html}.

\bibitem[Patrini et~al.(2017)Patrini, Rozza, Menon, Nock, and
  Qu]{patrini2017making}
Giorgio Patrini, Alessandro Rozza, Aditya~Krishna Menon, Richard Nock, and
  Lizhen Qu.
\newblock Making deep neural networks robust to label noise: A loss correction
  approach.
\newblock In \emph{Proceedings of the IEEE Conference on Computer Vision and
  Pattern Recognition}, pages 1944--1952, 2017.

\bibitem[Peres et~al.(2020)Peres, Lee, and Keich]{peres2020exactly}
Noah Peres, Andrew Lee, and Uri Keich.
\newblock Exactly computing the tail of the {Poisson-Binomial} distribution.
\newblock \emph{arXiv preprint arXiv:2004.07429}, 2020.
\newblock URL \url{https://arxiv.org/abs/2004.07429}.

\bibitem[Quadrianto et~al.(2009)Quadrianto, Smola, Caetano, and
  Le]{quadrianto2009estimating}
Novi Quadrianto, Alex~J Smola, Tib{\'e}rio~S Caetano, and Quoc~V Le.
\newblock Estimating labels from label proportions.
\newblock \emph{Journal of Machine Learning Research}, 10:\penalty0 2349--2374,
  2009.

\bibitem[Rakotomamonjy et~al.(2026)Rakotomamonjy, Vono, and
  Ralaivola]{rakotomamonjy2026particle}
Alain Rakotomamonjy, Maxime Vono, and Liva Ralaivola.
\newblock Particle flow for learning from label proportions.
\newblock In \emph{International Conference on Machine Learning}, 2026.
\newblock URL \url{https://icml.cc/virtual/2026/poster/63488}.

\bibitem[Sugiyama et~al.(2022)Sugiyama, Bao, Ishida, Lu, Sakai, and
  Niu]{sugiyama2022machine}
Masashi Sugiyama, Han Bao, Takashi Ishida, Nan Lu, Tomoya Sakai, and Gang Niu.
\newblock \emph{Machine learning from weak supervision: An empirical risk
  minimization approach}.
\newblock MIT Press, 2022.

\bibitem[Sun et~al.(2017)Sun, Sheldon, and O’Connor]{sun2017probabilistic}
Tao Sun, Dan Sheldon, and Brendan O’Connor.
\newblock A probabilistic approach for learning with label proportions applied
  to the {US} presidential election.
\newblock In \emph{Proceedings of the 2017 IEEE International Conference on
  Data Mining}, pages 445--454. IEEE, 2017.

\bibitem[Tokunaga et~al.(2020)Tokunaga, Iwana, Teramoto, Yoshizawa, and
  Bise]{tokunaga2020negative}
Hiroki Tokunaga, Brian~Kenji Iwana, Yuki Teramoto, Akihiko Yoshizawa, and Ryoma
  Bise.
\newblock Negative pseudo labeling using class proportion for semantic
  segmentation in pathology.
\newblock In \emph{European Conference on Computer Vision}, pages 430--446,
  2020.

\bibitem[Tsai and Lin(2020)]{tsai2020learning}
Kuen-Han Tsai and Hsuan-Tien Lin.
\newblock Learning from label proportions with consistency regularization.
\newblock In \emph{Asian Conference on Machine Learning}, pages 513--528, 2020.

\bibitem[Vinyals et~al.(2016)Vinyals, Blundell, Lillicrap, Wierstra,
  et~al.]{vinyals2016matching}
Oriol Vinyals, Charles Blundell, Timothy Lillicrap, Daan Wierstra, et~al.
\newblock Matching networks for one shot learning.
\newblock \emph{Advances in Neural Information Processing Systems}, 29, 2016.

\bibitem[Wang et~al.(2025)Wang, Ma, Xie, Niu, and Sugiyama]{wang2025rethinking}
Wei Wang, Tianhao Ma, Ming-Kun Xie, Gang Niu, and Masashi Sugiyama.
\newblock Rethinking consistent multi-label classification under inexact
  supervision.
\newblock \emph{arXiv preprint arXiv:2510.04091}, 2025.

\bibitem[Wei et~al.(2024)Wei, Shu, Cao, Wei, An, and Feng]{wei2024consistent}
Zixi Wei, Senlin Shu, Yuzhou Cao, Hongxin Wei, Bo~An, and Lei Feng.
\newblock Consistent multi-class classification from multiple unlabeled
  datasets.
\newblock In \emph{The Twelfth International Conference on Learning
  Representations (ICLR)}, 2024.

\bibitem[Yar{\i}kan et~al.(2026)Yar{\i}kan, {\"O}rer, Aky{\i}ld{\i}z,
  Ku{\c{s}}, Aydin, Palao{\u{g}}lu, {\.{I}}ncir, Baysal, {\"O}z{\c{c}}elik,
  Kiraz, and Kiraz]{yarikan2026large}
At{\i}f~Eren Yar{\i}kan, Can {\"O}rer, Volkan Aky{\i}ld{\i}z, Zeki Ku{\c{s}},
  Musa Aydin, Kerim~Erhan Palao{\u{g}}lu, Said {\.{I}}ncir, Kemal Baysal, Cemal
  {\"O}z{\c{c}}elik, Berna Kiraz, and Alper Kiraz.
\newblock A large-scale peripheral blood cell dataset for automated
  hematological analysis.
\newblock \emph{Scientific Data}, 13\penalty0 (1):\penalty0 417, 2026.
\newblock \doi{10.1038/s41597-026-06761-y}.
\newblock URL \url{https://doi.org/10.1038/s41597-026-06761-y}.

\bibitem[Yu et~al.(2014)Yu, Choromanski, Kumar, Jebara, and
  Chang]{yu2014learning}
Felix~X Yu, Krzysztof Choromanski, Sanjiv Kumar, Tony Jebara, and Shih-Fu
  Chang.
\newblock On learning from label proportions.
\newblock \emph{arXiv preprint arXiv:1402.5902}, 2014.

\bibitem[Yu et~al.(2018)Yu, Liu, Gong, and Tao]{yu2018learning}
Xiyu Yu, Tongliang Liu, Mingming Gong, and Dacheng Tao.
\newblock Learning with biased complementary labels.
\newblock In \emph{Proceedings of the European Conference on Computer Vision},
  pages 68--83, 2018.

\bibitem[Zhang et~al.(2022)Zhang, Wang, and Scott]{zhang2022learning}
Jianxin Zhang, Yutong Wang, and Clayton Scott.
\newblock Learning from label proportions by learning with label noise.
\newblock In \emph{Advances in Neural Information Processing Systems}, pages
  26933--26942, 2022.

\end{thebibliography}
\bibliographystyle{plainnat}

\clearpage
\appendix
\section{Notation and Terminology}
For convenience, \cref{tab:abbreviation-summary,tab:notation-summary} summarize the abbreviations and main notation used throughout the paper and this appendix.
\begin{table}[!htbp]
\caption{Notation used in the paper.}
\label{tab:notation-summary}
\centering
\footnotesize
\renewcommand{\arraystretch}{1.0}
\setlength{\tabcolsep}{4pt}
\begin{tabular}{@{}>{\raggedright\arraybackslash}p{0.28\linewidth}>{\raggedright\arraybackslash}p{0.66\linewidth}@{}}
\toprule
\textbf{Symbol} & \textbf{Meaning} \\
\midrule
$\sX$ & Input space. \\
$\sY=\set{1,\dots,q}$ & Label space. \\
$q$ & Number of classes. \\
$m$ & Bag size. \\
$\simplex{q}$ & The $(q-1)$-dimensional probability simplex $\set*{(p_1,\ldots,p_q)\in\R^q \given p_k\ge 0,\ \sum_{k=1}^q p_k=1}$. \\
$x_j \in \sX$ & The $j$-th instance in a bag. \\
$y_j \in \sY$ & Latent class label of $x_j$. \\
$\bag=(x_1,\dots,x_m)\in \sX^m$ & Bag of $m$ instances. \\
$\vy=(y_1,\dots,y_m)\in \sY^m$ & Latent label sequence associated with $\bag$. \\
$\sY^{m/\propvec}$ & Set of label sequences in $\sY^m$ whose empirical class proportions equal $\propvec$. \\
$\indicator{\cdot}:\set{\text{statements}}\to\set{0,1}$ & Indicator function, which returns $1$ when its argument is true and $0$ otherwise. \\
$\prop_k \in [0,1]$ & Proportion of class $k$ in a bag, defined by $\prop_k=\frac{1}{m}\sum_{j=1}^m \indicator{y_j=k}$. \\
$\propvec=(\prop_1,\dots,\prop_q)\in \simplex{q}$ & Bag-level class-proportion vector. \\
$\vz = m\propvec \in \set{0,\dots,m}^q$ & Bag-level class-count vector. \\
$\ve^{y_j}\in \set{0,1}^q \subset \R^q$ & One-hot basis vector corresponding to label $y_j$. \\
$p:\sX\times\sY\to\R_+$ & Joint data distribution over instance-label pairs. \\
$p(\cdot\mid x):\sY\to[0,1]$ & Class-posterior map for a fixed instance $x\in\sX$. \\
\midrule
$\sF \subseteq \set*{\classifier:\sX\to\simplex{q}}$ & Hypothesis class of score-based classifiers. \\
$\classifier:\sX\to\simplex{q}$ & Score-based classifier, with $\classifier\in\sF$. \\
$\classifier^*:\sX\to\simplex{q}$ & Bayes-optimal classifier. \\
$p_j=\classifier(x_j)\in\simplex{q}$ & Predicted class-probability vector for instance $x_j$. \\
\midrule
$\ell:\simplex{q}\times\sY\to\R_+$ & Classification-calibrated surrogate loss used in ordinary supervised learning. \\
$\M$ & Generic normalized population moment object; for example, $\M\order{1}_{\propvec}$ is the normalized first-order population moment and $\M\order{2}$ denotes the normalized bag-level second-order moment matrix. \\
$(m)_r$ & Falling factorial $m(m-1)\cdots(m-r+1)$. \\
$[m]_{\neq}^r$ & Set of ordered $r$-tuples of distinct indices from $[m]$. \\
$\Mhat\order{r}_{\propvec}\in(\R^q)^{\otimes r}$ & Normalized $r$-th order factorial moment tensor computed from $\propvec$. \\
$\Mhat\order{r}_\classifier:\sX^m\to(\R^q)^{\otimes r}$ & Classifier-induced normalized $r$-th order factorial moment map; $\Mhat\order{r}_\classifier(\bag)$ is its value on bag $\bag$. \\
$\Loss\order{r}$ & Order-$r$ discrepancy loss between normalized target and classifier-induced moments; for $r=1$ it acts on class-proportion vectors, and for general $r$ it acts on tensors in $(\R^q)^{\otimes r}$. \\
$R_\MO\order{s}:\sF\to\R_+$ & Population multi-order normalized moment-matching risk using moments up to order $s$. \\
$\widehat{R}_\MO\order{s}:\sF\to\R_+$ & Empirical multi-order normalized moment-matching risk using moments up to order $s$. \\
$s\in\set{1,\dots,m}$ & Maximum moment order used for training. \\
$\lambda_r > 0$ & Weight of the $r$-th order term, with $\sum_{r=1}^s \lambda_r=1$. \\
\midrule
$\Rademacher$ & Rademacher-complexity symbol. \\
$\Rademacher_n(\sF_c)$ & Expected Rademacher complexity of the coordinate function class $\sF_c$. \\
$\vk$ & Multiplicity vector recording how often each class index appears in an ordered index tuple. \\
$\sK_{q,r}$ & Compressed multiplicity-state space at order $r$, defined by $\sK_{q,r}=\set*{\vk\in\N^q \given \sum_{c=1}^q k_c=r}$. \\
$N_r$ & Cardinality of the compressed multiplicity-state space, namely $N_r=\abs*{\sK_{q,r}}=\binom{q+r-1}{r}$. \\
$A_t\order{d}:\sK_{q,d}\to\R$ & Coefficient map in the truncated generating polynomial; $A_t\order{d}(\vk)$ is the coefficient associated with state $\vk$. \\
\bottomrule
\end{tabular}
\end{table}

\begin{table}[htbp]
\caption{Abbreviations used in the paper.}
\label{tab:abbreviation-summary}
\centering
\small
\renewcommand{\arraystretch}{1.1}
\setlength{\tabcolsep}{4pt}
\begin{tabular}{@{}>{\raggedright\arraybackslash}p{0.22\linewidth}>{\raggedright\arraybackslash}p{0.72\linewidth}@{}}
\toprule
\textbf{Abbreviation} & \textbf{Meaning} \\
\midrule
WSL & Weakly supervised learning \\
LLP & Learning from label proportions \\
PMD & Poisson multinomial distribution \\
MM & Moment matching \\
LLP-MM & LLP via Moment Matching \\
\bottomrule
\end{tabular}
\end{table}

\section{First-Order Moment Matching}
\label{first_order_matching}
Based on the connection between LLP and the PMD, we propose a general framework which matches the empirical factorial moments with classifier-induced moments. We present the first-order moment matching approach, which has been referred to as proportion matching in prior work~\citep{yu2014learning,ardehaly2017co}. Specifically, let $\M\order{1}_\propvec = \bracks{M\order{1}_1, M\order{1}_2, \ldots, M\order{1}_q}$ denote the normalized first-order moment of $\vz$. Then, we have
\begin{equation}
\label{eq:moment_1order}
M\order{1}_k
=
\frac{1}{m}
\sum_{j=1}^m
\E_{p(y_j\mid x_j)}\bracks*{\indicator{y_j = k}}
=
\frac{1}{m}
\sum_{j=1}^m
p(y_j = k \mid x_j).
\end{equation}
However, since the instance-level posterior probabilities are unknown, we approximate $\M\order{1}_\propvec$ with the \emph{first-order empirical moment} $\Mhat\order{1}_\propvec$, where the $k$-th element is
\begin{equation}
  \bracks*{\Mhat\order{1}_\propvec}_k
=
\frac{1}{m}
\sum_{j=1}^m \bracks{\indicator{y_j = k}}
=
\frac{z_k}{m}
=
\prop_k,
\end{equation}
and $z_k$ is the $k$-th dimension of $\vz$. Apart from empirical moments, we can also use the classifier outputs to approximate $\M\order{1}_\propvec$.
Specifically, we introduce 
\begin{equation}
\Mhat\order{1}_\classifier(\bag)
=
\frac{1}{m}\sum_{j=1}^m \classifier(x_j)
\end{equation}
to represent the first-order classifier-induced moment. Finally, we introduce the \emph{first-order moment matching risk} to match the first-order empirical factorial moments with the first-order classifier-induced moment:
\begin{equation}
\label{eq:pm-risk}
R_\MO\order{1}(\classifier)
=
\E_{p(\bag, \propvec)}
\bracks*{
\Loss\order{1}\!\parens*{
\propvec,
\Mhat\order{1}_\classifier(\bag)
}
},
\end{equation}
where $\Loss\order{1}:\ \simplex{q} \times \simplex{q} \to [0,+\infty]$ measures the discrepancy between the true and predicted class proportions. 
Notably, the formulation in \cref{eq:pm-risk} is the same as the objective of the popular proportion matching strategy~\citep{tsai2020learning,ma2025forming,Liu2019Learning} which uses the observed proportions as targets for the predicted mean proportions. Our moment matching framework provides a new perspective for understanding the proportion matching strategy based on the PMD.

\section{Derivation of the Second-Order Moment from Label Proportions}
\label{app:second-order-derivation}

Recall that the bag-level second-order moment matrix is
\begin{equation}
\M\order{2}
=
\frac{1}{m(m-1)}
\sum_{i \neq j}
\ve^{y_i}(\ve^{y_j})^\top.
\end{equation}
We show that each entry of $\M\order{2}$ can be computed directly from the label proportion vector $\propvec$.

For any $c, d \in \set{1, \dots, q}$, the $(c, d)$-th entry of $\M\order{2}$ is
\begin{align}
M\order{2}_{cd}
&=
\frac{1}{m(m-1)}
\sum_{i \neq j}
\bracks*{\ve^{y_i}}_c
\bracks*{\ve^{y_j}}_d
\\
&=
\frac{1}{m(m-1)}
\sum_{i \neq j}
\indicator{y_i = c} \indicator{y_j = d}.
\label{eq:m2-entry-start}
\end{align}
Let
\begin{equation}
n_c
\defeq
\sum_{i=1}^m \indicator{y_i = c}
=
m\prop_c
\end{equation}
denote the number of instances in the bag belonging to class $c$.
Then we consider two cases.

\paragraph{Case 1: $c \neq d$.}

In this case, the constraint $i \neq j$ is automatically satisfied whenever
$\indicator{y_i = c} \indicator{y_j = d} = 1$, because one instance cannot simultaneously belong to two different classes.
Hence,
\begin{align}
\sum_{i \neq j} \indicator{y_i = c} \indicator{y_j = d}
&=
\sum_{i=1}^m \sum_{j=1}^m \indicator{y_i = c} \indicator{y_j = d}
\\
&=
\sum_{i=1}^m \sum_{j=1}^m \indicator{y_i = c}\indicator{y_j = d}
\\
&=
\parens*{\sum_{i=1}^m \indicator{y_i = c}}
\parens*{\sum_{j=1}^m \indicator{y_j = d}} 
\\
&=
n_c n_d.
\label{eq:m2-case-diff}
\end{align}
Substituting $n_c = m\prop_c$ and $n_d = m\prop_d$ into \cref{eq:m2-entry-start}, we obtain
\begin{equation}
M\order{2}_{cd}
=
\frac{n_c n_d}{m(m-1)}
=
\frac{m\prop_c\, m\prop_d}{m(m-1)},
\qquad c\neq d.
\end{equation}

\paragraph{Case 2: $c=d$.}
Now the constraint $i\neq j$ matters.
We have
\begin{align}
\sum_{i \neq j} \indicator{y_i = c} \indicator{y_j = c}
&=
\parens*{\sum_{i=1}^m \indicator{y_i = c}}
\parens*{\sum_{j=1}^m \indicator{y_j = c}}
-
\sum_{i=1}^m \indicator{y_i = c}
\\
&=
n_c^2 - n_c
\\
&=
n_c(n_c - 1).
\label{eq:m2-case-same}
\end{align}
Therefore,
\begin{equation}
M\order{2}_{cc}
=
\frac{n_c(n_c - 1)}{m(m-1)}
=
\frac{m\prop_c(m\prop_c - 1)}{m(m-1)}.
\end{equation}

Combining the two cases, we obtain
\begin{equation}
M\order{2}_{cd}
=
\frac{n_c\parens*{n_d - \indicator{c = d}}}{m(m-1)}
=
\frac{m\prop_c\parens*{m\prop_d-\indicator{c=d}}}{m(m-1)},
\qquad c, d \in \set{1, \dots, q}.
\end{equation}
This proves that the second-order moment matrix is completely determined by the observed label proportion vector $\propvec$.
\qed

\section{Proof of \cref{prop:mo-nesting}}
\begin{theorem}
\label{thm:mo-bayes-moments}
Assume that the hypothesis class $\sF$ is expressive enough to contain the Bayes-optimal classifier, i.e.,
\(
\classifier^* \in \sF.
\)
For each $r \in \set{1, \dots, s}$, assume that for any
$\mathcal T_r$-valued random variable $Z$,
its expectation $\E[Z]$ is the unique solution of
\(
\min_{v \in \mathcal T_r}
\E\bracks*{\Loss\order{r}(Z,v)}.
\)
Then, for any minimizer $\classifier \in \sF_\MO\order{s}$ of \cref{eq:multi-order-risk}, for $p(\bag)$-almost every bag $\bag$ and every $r \in \set{1,\dots,s}$ such that $\lambda_r>0$, we have
\begin{equation}
\label{eq:mo-theorem-main}
\Mhat\order{r}_\classifier(\bag)
=
\Mhat\order{r}_{\classifier^*}(\bag)
\in (\R^q)^{\otimes r}.
\end{equation}
That is, every minimizer of \cref{eq:multi-order-risk} exactly matches the moment tensors induced by the Bayes-optimal classifier for all positively weighted orders, for $p(\bag)$-almost every bag $\bag$.
\end{theorem}
\cref{thm:mo-bayes-moments} shows that incorporating multi-order moments as training supervision enforces agreement with the Bayes-optimal classifier-induced moments at correspondingly higher orders.
\begin{proof}[Proof of Theorem~\ref{thm:mo-bayes-moments}]
For each order $r\in\set{1,\dots,s}$, define the random tensor
\begin{equation}
Z_r \defeq \Mhat\order{r}_{\propvec}.
\end{equation}
We first show that, conditioned on a bag $\bag=(x_1,\dots,x_m)$, its conditional expectation is exactly the classifier-induced moment tensor of the Bayes-optimal classifier:
\begin{equation}
\label{eq:conditional-moment-equals-bayes}
\E\bracks*{Z_r \mid \bag}
=
\Mhat\order{r}_{\classifier^*}(\bag).
\end{equation}
Indeed, by \cref{eq:rth-factorial-tensor-def},
\begin{equation}
Z_r
=
\frac{1}{(m)_r}
\sum_{(j_1,\dots,j_r)\in[m]_{\neq}^r}
\ve^{y_{j_1}}\otimes\cdots\otimes \ve^{y_{j_r}}.
\end{equation}
Taking conditional expectation given $\bag$ and using both the conditional independence of $y_1,\dots,y_m$ given $\bag$ and the conditional marginal identities $p(y_j\mid\bag)=p(y_j\mid x_j)$ assumed in \cref{eq:seq-post}, we obtain
\begin{equation}
\E\bracks*{
\ve^{y_{j_1}}\otimes\cdots\otimes \ve^{y_{j_r}}
\mid \bag
}
=
\bigotimes_{\ell=1}^r \E[\ve^{y_{j_\ell}}\mid x_{j_\ell}].
\end{equation}
By definition of the Bayes-optimal classifier,
\begin{equation}
\E[\ve^{y_j}\mid x_j]
=
p(\cdot\mid x_j)
=
\classifier^*(x_j).
\end{equation}
Hence
\begin{equation}
\E\bracks*{
\ve^{y_{j_1}}\otimes\cdots\otimes \ve^{y_{j_r}}
\mid \bag
}
=
\classifier^*(x_{j_1})\otimes\cdots\otimes \classifier^*(x_{j_r}),
\end{equation}
and therefore
\begin{equation}
\E[Z_r\mid \bag]
=
\frac{1}{(m)_r}
\sum_{(j_1,\dots,j_r)\in[m]_{\neq}^r}
\classifier^*(x_{j_1})\otimes\cdots\otimes \classifier^*(x_{j_r})
=
\Mhat\order{r}_{\classifier^*}(\bag),
\end{equation}
which proves \cref{eq:conditional-moment-equals-bayes}.

Now let $\classifier\in\sF_\MO\order{s}$ be any minimizer of \cref{eq:multi-order-risk}. For each $r$, define
\begin{equation}
V_r(\bag)\defeq \Mhat\order{r}_\classifier(\bag).
\end{equation}
Then the multi-order risk can be written as
\begin{equation}
R_\MO\order{s}(\classifier)
=
\E_{\bag}\bracks*{
\sum_{r=1}^s
\lambda_r\,
\E\bracks*{
\Loss\order{r}\parens*{Z_r,V_r(\bag)}
\mid \bag
}
}.
\end{equation}

Fix $r$, and define
\begin{equation}
g_r(v,\bag)
\defeq
\E\bracks*{\Loss\order{r}(Z_r,v)\mid \bag}.
\end{equation}
Applying the loss assumption of the theorem to the conditional distribution of $Z_r$ given $\bag$, the unique minimizer of $v\mapsto g_r(v,\bag)$ is the conditional expectation $\E[Z_r\mid\bag]$. By \cref{eq:conditional-moment-equals-bayes}, this unique minimizer is exactly $\Mhat\order{r}_{\classifier^*}(\bag)$. Therefore, for every $v\in\mathcal T_r$,
\begin{equation}
g_r(v,\bag)
\ge
g_r\parens*{\Mhat\order{r}_{\classifier^*}(\bag),\bag},
\end{equation}
with equality if and only if
\begin{equation}
v=\Mhat\order{r}_{\classifier^*}(\bag).
\end{equation}

Since $\classifier$ is a minimizer of \cref{eq:multi-order-risk} and $\classifier^*$ is feasible, we also have
\begin{equation}
R_\MO\order{s}(\classifier)
\le
R_\MO\order{s}(\classifier^*).
\end{equation}
On the other hand, by the pointwise lower bound above,
\begin{align}
R_\MO\order{s}(\classifier)
-
R_\MO\order{s}(\classifier^*)
&=
\E_{\bag}\bracks*{
\sum_{r=1}^s
\lambda_r\,
\bracks*{
 g_r\parens*{V_r(\bag),\bag}
 -
 g_r\parens*{\Mhat\order{r}_{\classifier^*}(\bag),\bag}
 }
}
\\
&\ge 0.
\end{align}
Therefore,
\begin{equation}
\E_{\bag}\bracks*{
\sum_{r=1}^s
\lambda_r\,
\bracks*{
 g_r\parens*{V_r(\bag),\bag}
 -
 g_r\parens*{\Mhat\order{r}_{\classifier^*}(\bag),\bag}
 }
}
=0.
\end{equation}
The integrand is nonnegative almost surely, so it must vanish almost surely. Hence, for every $r$ with $\lambda_r>0$,
\begin{equation}
g_r\parens*{V_r(\bag),\bag}
=
g_r\parens*{\Mhat\order{r}_{\classifier^*}(\bag),\bag}
\end{equation}
for $p(\bag)$-almost every bag $\bag$. By the uniqueness of the minimizer of $g_r(\cdot,\bag)$, this implies
\begin{equation}
V_r(\bag)=\Mhat\order{r}_{\classifier^*}(\bag)
\end{equation}
for $p(\bag)$-almost every bag $\bag$ and every $r$ such that $\lambda_r>0$. Equivalently,
\begin{equation}
\Mhat\order{r}_\classifier(\bag)
=
\Mhat\order{r}_{\classifier^*}(\bag),
\qquad
\text{for $p(\bag)$-almost every bag $\bag$ and every $r$ such that $\lambda_r>0$.}
\end{equation}
This proves that any minimizer $\classifier$ induces exactly the same moment tensors as $\classifier^*$ for all orders with positive weight.
\end{proof}
\begin{proof}[Proof of \cref{prop:mo-nesting}]
By \cref{thm:mo-bayes-moments}, any \(\classifier \in \sF_\MO\order{s_2}\)
matches the Bayes-optimal classifier-induced moment tensors for every order
\(r \in \set{1,\dots,s_2}\), since \(\lambda_r>0\) for all such \(r\). In
particular, it matches these moment tensors for every order
\(r \in \set{1,\dots,s_1}\). Therefore, \(\classifier\) also satisfies the
moment-matching conditions required for membership in \(\sF_\MO\order{s_1}\).
Hence,
\[
\sF_\MO\order{s_2}
\subseteq
\sF_\MO\order{s_1}.
\]
\end{proof}
\section{Generalization Analysis}

\begin{definition}[Loss-composed hypothesis class and its Rademacher complexity]
\label{def:mo_rade}
For each order $r \in \set{1, \dots, s}$, define the bag-level hypothesis class
\begin{equation}
\sG\order{r}
\defeq
\set{
(\bag, \propvec) \mapsto \Loss\order{r}
\parens*{
\Mhat\order{r}_\propvec,
\Mhat\order{r}_\classifier(\bag)
}
:\ \classifier \in \sF
}.
\end{equation}
Its expected Rademacher complexity is defined as
\begin{equation}
\Rademacher\order{r}_n(\sG\order{r})
\defeq
\E_\sD
\E_\vsigma
\bracks*{
\sup_{\classifier \in \sF}
\frac1n
\sum_{i=1}^n
\sigma_i\,
\Loss\order{r}
\parens*{
\Mhat\order{r}_{\propvec_i},
\Mhat\order{r}_\classifier(\bag_i)
}
},
\end{equation}
where $\sD = \set{(\bag_i, \propvec_i)}_{i=1}^n$ consists of $n$ i.i.d.~samples from $p(\bag, \propvec)$, and $\sigma_1, \dots, \sigma_n$ are i.i.d.\ Rademacher variables.
\end{definition}

\begin{theorem}[Generic generalization bound via loss-composed classes]
\label{thm:mo_gen_lossclass}
Assume that for each $r \in \set{1, \dots, s}$,
\begin{equation}
0 \le
\Loss\order{r}
\parens*{
\Mhat\order{r}_\propvec,
\Mhat\order{r}_\classifier(\bag)
}
\le B_r
\end{equation}
for all $(\bag, \propvec)$ and all $\classifier \in \sF$.
Then, for any $\delta \in (0, 1)$, with probability at least $1 - \delta$,
\begin{equation}
\label{eq:generic_lossclass_bound}
\sup_{\classifier \in \sF}
\abs{
R_\MO\order{s}(\classifier)
-
\widehat{R}_\MO\order{s}(\classifier)
}
\le
2 \sum_{r=1}^s \lambda_r\,\Rademacher\order{r}_n(\sG\order{r})
+
\parens*{\sum_{r=1}^s \lambda_r B_r}
\sqrt{\frac{\log(2/\delta)}{2n}}.
\end{equation}
\end{theorem}

\begin{proof}
Define the aggregated loss class
\begin{equation}
\sH
\defeq
\set{
(\bag, \propvec)
\mapsto
\sum_{r=1}^s \lambda_r
\Loss\order{r}
\parens*{
\Mhat\order{r}_\propvec,
\Mhat\order{r}_\classifier(\bag)
}
:\ \classifier \in \sF
}.
\end{equation}
Then
\begin{equation}
R_\MO\order{s}(\classifier)
=
\E[h_\classifier(\bag, \propvec)],
\qquad
\widehat{R}_\MO\order{s}(\classifier)
=
\frac1n \sum_{i=1}^n h_\classifier(\bag_i,\propvec_i),
\end{equation}
where $h_\classifier \in \sH$.

Define the two one-sided deviations
\begin{align*}
U(\sD)
&\defeq
\sup_{\classifier \in \sF}
\parens*{
R_\MO\order{s}(\classifier)
-
\widehat{R}_\MO\order{s}(\classifier)
},
\\
V(\sD)
&\defeq
\sup_{\classifier \in \sF}
\parens*{
\widehat{R}_\MO\order{s}(\classifier)
-
R_\MO\order{s}(\classifier)
}.
\end{align*}
Applying the standard one-sided symmetrization inequality~\citep{mohri2018foundations} separately in the two directions gives
\begin{equation}
\E_\sD U(\sD) \le 2\,\Rademacher_n(\sH),
\qquad
\E_\sD V(\sD) \le 2\,\Rademacher_n(\sH),
\end{equation}
where
\begin{equation}
\Rademacher_n(\sH)
=
\E_\sD \E_\vsigma
\bracks*{
\sup_{\classifier \in \sF}
\frac1n \sum_{i=1}^n \sigma_i\, h_\classifier(\bag_i,\propvec_i)
}.
\end{equation}
The reverse direction uses the symmetry of the Rademacher signs: replacing $\vsigma$ by $-\vsigma$ leaves their distribution unchanged. Thus both inequalities use the same complexity defined above, without an absolute value inside the supremum.
Using the definition of $h_\classifier$ and the subadditivity of the supremum,
\begin{align*}
\Rademacher_n(\sH)
&=
\E_\sD \E_\vsigma
\bracks*{
\sup_{\classifier \in \sF}
\frac1n \sum_{i=1}^n \sigma_i
\sum_{r=1}^s \lambda_r
\Loss\order{r}
\parens*{
\Mhat\order{r}_{\propvec_i},
\Mhat\order{r}_\classifier(\bag_i)
}
}
\\
&\le
\sum_{r=1}^s \lambda_r
\E_\sD \E_\vsigma
\bracks*{
\sup_{\classifier \in \sF}
\frac1n \sum_{i=1}^n \sigma_i
\Loss\order{r}
\parens*{
\Mhat\order{r}_{\propvec_i},
\Mhat\order{r}_\classifier(\bag_i)
}
}
\\
&=
\sum_{r=1}^s \lambda_r\, \Rademacher\order{r}_n(\sG\order{r}).
\end{align*}
Hence,
\begin{equation}
\max\set{\E_\sD U(\sD),\E_\sD V(\sD)}
\le
2 \sum_{r=1}^s \lambda_r\, \Rademacher\order{r}_n(\sG\order{r}).
\end{equation}

Now let
\begin{equation}
C \defeq \sum_{r=1}^s \lambda_r B_r.
\end{equation}
For every $\classifier \in \sF$, the aggregated loss
\begin{equation}
h_\classifier(\bag,\propvec)
=
\sum_{r=1}^s \lambda_r
\Loss\order{r}
\parens*{
\Mhat\order{r}_\propvec,
\Mhat\order{r}_\classifier(\bag)
}
\end{equation}
satisfies $0 \le h_\classifier \le C$.
Therefore, if $\sD$ and $\sD'$ differ in one bag-proportion pair $(\bag_i,\propvec_i)$, then
\begin{equation}
\abs*{U(\sD)-U(\sD')} \le \frac{C}{n},
\qquad
\abs*{V(\sD)-V(\sD')} \le \frac{C}{n}.
\end{equation}
The independent sample units are the $n$ entire bag-proportion pairs. Applying McDiarmid's inequality~\citep{mohri2018foundations} separately to $U(\sD)$ and $V(\sD)$, with failure probability $\delta/2$ for each, gives
\[
\Pr\parens*{
W(\sD) > \E_\sD W(\sD)
+ C\sqrt{\frac{\log(2/\delta)}{2n}}
}
\le \frac{\delta}{2},
\qquad W\in\set{U,V}.
\]
By a union bound and the two symmetrization inequalities, with probability at least $1-\delta$, both directions simultaneously satisfy
\begin{equation}
\max\set{U(\sD),V(\sD)}
\le
2\,\Rademacher_n(\sH) +
C\sqrt{\frac{\log(2/\delta)}{2n}}.
\end{equation}
Finally, using
\[
\sup_{\classifier \in \sF}
\abs*{
R_\MO\order{s}(\classifier)
-
\widehat{R}_\MO\order{s}(\classifier)
}
= \max\set{U(\sD),V(\sD)},
\]
together with the subadditivity bound for $\Rademacher_n(\sH)$ and the definition of $C$, proves~\cref{eq:generic_lossclass_bound}.
\end{proof}

\begin{definition}[Coordinate-wise Rademacher complexity of $\sF$]
\label{def:coord_rade}
For each position $j \in [m]$ and class $c \in [q]$, define the scalar-valued class
\begin{equation}
\sF_{j, c}
\defeq
\set{
\bag = (x_1, \dots, x_m) \mapsto [\classifier(x_j)]_c
:\ \classifier \in \sF
}.
\end{equation}
Its expected Rademacher complexity is
\begin{equation}
\Rademacher_n(\sF_{j, c})
\defeq
\E_{\set{\bag_i}_{i=1}^n}
\E_{\set{\sigma_i}_{i=1}^n}
\bracks*{
\sup_{\classifier \in \sF}
\frac1n \sum_{i=1}^n \sigma_i\,[\classifier(x_{ij})]_c
}.
\end{equation}
\end{definition}

\begin{lemma}[Lipschitz property of the moment map]
\label{lem:moment_lipschitz}
For any $r \in \set{1, \dots, m}$, any bag $\bag = (x_1, \dots, x_m)$, and any
$\classifier, \classifier' \in \sF$, we have
\begin{equation}
\label{eq:moment_lipschitz}
\norm{
\Mhat\order{r}_\classifier(\bag)
-
\Mhat\order{r}_{\classifier'}(\bag)
}_F
\le
\frac{r}{m}
\sum_{j=1}^m
\norm*{\classifier(x_j) - \classifier'(x_j)}_2
\le
\frac{r}{m}
\sum_{j=1}^m \sum_{c=1}^q
\abs*{[\classifier(x_j)]_c - [\classifier'(x_j)]_c},
\end{equation}
where $\norm{\cdot}_F$ denotes the Frobenius norm of an order-$r$ tensor, i.e., for $T = (T_{i_1, \dots, i_r}) \in \R^{q \times \cdots \times q}$,
\begin{equation}
\norm{T}_F
\defeq
\parens*{
\sum_{i_1, \dots, i_r} T_{i_1, \dots, i_r}^2
}^{1/2}.
\end{equation}
\end{lemma}
\begin{proof}
Let $u_j \defeq \classifier(x_j) \in \simplex{q}$ and $v_j \defeq \classifier'(x_j) \in \simplex{q}$.
For any ordered $r$-tuple $(j_1, \dots, j_r) \in [m]_{\neq}^r$, write
\begin{equation}
U \defeq u_{j_1} \otimes \cdots \otimes u_{j_r},
\qquad
V \defeq v_{j_1} \otimes \cdots \otimes v_{j_r}.
\end{equation}
A telescoping expansion gives
\begin{equation}
U-V
=
\sum_{t=1}^r
v_{j_1} \otimes \cdots \otimes v_{j_{t-1}}
\otimes (u_{j_t} - v_{j_t})
\otimes u_{j_{t+1}} \otimes \cdots \otimes u_{j_r}.
\end{equation}
Using
\begin{equation}
\norm{a_1 \otimes \cdots \otimes a_r}_F
=
\prod_{\ell=1}^r \norm{a_\ell}_2
\end{equation}
and the fact that $\norm{u_j}_2 \le \norm{u_j}_1=1$ and $\norm{v_j}_2 \le \norm{v_j}_1=1$ for all $j$, we obtain
\begin{equation}
\norm{U - V}_F
\le
\sum_{t=1}^r \norm{u_{j_t} - v_{j_t}}_2.
\end{equation}
Hence
\begin{align*}
\norm{
\Mhat\order{r}_\classifier(\bag)
-
\Mhat\order{r}_{\classifier'}(\bag)
}_F
&\le
\frac{1}{(m)_r}
\sum_{(j_1, \dots, j_r) \in [m]_{\neq}^r}
\sum_{t=1}^r \norm{u_{j_t} - v_{j_t}}_2
\\
&=
\frac{r(m-1)_{r-1}}{(m)_r}
\sum_{j=1}^m \norm{u_j - v_j}_2
\\
&=
\frac{r}{m}
\sum_{j=1}^m \norm{u_j - v_j}_2.
\end{align*}
The second inequality in~\cref{eq:moment_lipschitz} follows from
$\norm{a}_2 \le \norm{a}_1 = \sum_{c=1}^q \abs{a_c}$.
\end{proof}

\begin{lemma}[Coordinate-wise contraction]
\label{lem:coord_contraction_weak}
Let $T \subset (\R^d)^n$.
For each $i \in [n]$, let $\Phi_i: \R^d \to \R$ satisfy $\Phi_i(\vzero)=0$ and
\begin{equation}
|\Phi_i(u)-\Phi_i(v)| \le L \|u-v\|_1
\qquad
\text{for all } u,v\in\R^d.
\end{equation}
Then
\begin{equation}
\E_\vsigma\bracks*{
\sup_{t\in T}\frac1n\sum_{i=1}^n \sigma_i \Phi_i(t_i)
}
\le
\sqrt{2d}\,L
\sum_{\ell=1}^d
\E_{\vtau^{(\ell)}}\bracks*{
\sup_{t\in T}\frac1n\sum_{i=1}^n \tau_i^{(\ell)} t_{i,\ell}
}.
\end{equation}
\end{lemma}

\begin{proof}
Since $\|x\|_1 \le \sqrt d\,\|x\|_2$ for every $x\in\R^d$, the assumption implies
\begin{equation}
|\Phi_i(u)-\Phi_i(v)|
\le
L\sqrt d\,\|u-v\|_2
\qquad
\text{for all } u,v\in\R^d.
\end{equation}
Hence each $\Phi_i$ is $L\sqrt d$-Lipschitz with respect to the Euclidean norm.

Applying the vector contraction inequality~\citep{maurer2016vector} with respect to $\|\cdot\|_2$ yields
\begin{equation}
\E_\vsigma\bracks*{
\sup_{t\in T}\frac1n\sum_{i=1}^n \sigma_i \Phi_i(t_i)
}
\le
\sqrt{2d}\,L\,
\E_\vepsilon\bracks*{
\sup_{t\in T}\frac1n\sum_{i=1}^n\sum_{\ell=1}^d \epsilon_{i,\ell} t_{i,\ell}
},
\end{equation}
where $\epsilon_{i,\ell}$ are i.i.d.~Rademacher variables.

Finally, by subadditivity of the supremum,
\begin{align}
\E_\vepsilon\bracks*{
\sup_{t\in T}\frac1n\sum_{i=1}^n\sum_{\ell=1}^d \epsilon_{i,\ell} t_{i,\ell}
}
&\le
\sum_{\ell=1}^d
\E_\vepsilon\bracks*{
\sup_{t\in T}\frac1n\sum_{i=1}^n \epsilon_{i,\ell} t_{i,\ell}
}
\\
&=
\sum_{\ell=1}^d
\E_{\vtau^{(\ell)}}\bracks*{
\sup_{t\in T}\frac1n\sum_{i=1}^n \tau_i^{(\ell)} t_{i,\ell}
}.
\end{align}
This proves the claim.
\end{proof}
\begin{proposition}[Reduction to the base class $\sF$]
\label{prop:reduce_to_F}
Assume that for each $r \in \set{1, \dots, s}$, the loss
\begin{equation}
B \mapsto \Loss\order{r}(A, B)
\end{equation}
is $L_r$-Lipschitz with respect to the Frobenius norm, uniformly over all
admissible target tensors $A$. Then
\begin{equation}
\label{eq:lossclass_to_F}
\Rademacher\order{r}_n(\sG\order{r})
\le
\sqrt{2m}\,\frac{L_r\, r}{m}
\sum_{j=1}^m \sum_{c=1}^q
\Rademacher_n(\sF_{j, c}).
\end{equation}
\end{proposition}
\begin{proof}
Fix a sample $\set{(\bag_i,\propvec_i)}_{i=1}^n$. For a bag
$\bag=(x_1,\dots,x_m)$ and a classifier $\classifier$, write
\[
\vu_j = \classifier(x_j) \in \Delta^{q-1},
\qquad
\vu = (\vu_1,\dots,\vu_m) \in (\R^q)^m.
\]
We first show that the classifier-induced order-$r$ factorial moment tensor is
$r/\sqrt{m}$-Lipschitz as a function of $\vu$ with respect to the Euclidean norm.
For $\vu,\vv\in(\Delta^{q-1})^m$, define
\[
\Mhat\order{r}(\vu)
=
\frac{1}{(m)_r}
\sum_{(j_1,\dots,j_r)\in[m]_{\neq}^r}
\vu_{j_1}\otimes \cdots \otimes \vu_{j_r},
\]
where $[m]_{\neq}^r$ denotes the set of ordered $r$-tuples of distinct indices.
By the triangle inequality and a telescoping expansion of tensor products,
\begin{align}
\left\|
\Mhat\order{r}(\vu)-\Mhat\order{r}(\vv)
\right\|_F
&\le
\frac{1}{(m)_r}
\sum_{(j_1,\dots,j_r)\in[m]_{\neq}^r}
\left\|
\vu_{j_1}\otimes\cdots\otimes\vu_{j_r}
-
\vv_{j_1}\otimes\cdots\otimes\vv_{j_r}
\right\|_F                                                    \nonumber\\
&\le
\frac{1}{(m)_r}
\sum_{(j_1,\dots,j_r)\in[m]_{\neq}^r}
\sum_{\ell=1}^r
\left\|
\vu_{j_\ell}-\vv_{j_\ell}
\right\|_2 .
\label{eq:tensor_lip_step}
\end{align}
The second inequality uses the fact that
$\|\vu_j\|_2\le 1$ and $\|\vv_j\|_2\le 1$ for probability vectors.
For each fixed index $j\in[m]$ and each position $\ell\in\set{1,\dots,r}$,
the number of ordered tuples $(j_1,\dots,j_r)\in[m]_{\neq}^r$ satisfying
$j_\ell=j$ is $(m-1)_{r-1}$. Therefore,
\begin{align}
\frac{1}{(m)_r}
\sum_{(j_1,\dots,j_r)\in[m]_{\neq}^r}
\sum_{\ell=1}^r
\left\|
\vu_{j_\ell}-\vv_{j_\ell}
\right\|_2
&=
\frac{r(m-1)_{r-1}}{(m)_r}
\sum_{j=1}^m
\left\|
\vu_j-\vv_j
\right\|_2                                                  \nonumber\\
&=
\frac{r}{m}
\sum_{j=1}^m
\left\|
\vu_j-\vv_j
\right\|_2 .
\label{eq:l1_over_instances}
\end{align}
Applying Cauchy--Schwarz over the $m$ bag instances gives
\begin{equation}
\sum_{j=1}^m
\left\|
\vu_j-\vv_j
\right\|_2
\le
\sqrt{m}
\left(
\sum_{j=1}^m
\left\|
\vu_j-\vv_j
\right\|_2^2
\right)^{1/2}
=
\sqrt{m}\,
\|\vu-\vv\|_2 .
\end{equation}
Combining this with \cref{eq:tensor_lip_step,eq:l1_over_instances}, we obtain
\begin{equation}
\left\|
\Mhat\order{r}(\vu)-\Mhat\order{r}(\vv)
\right\|_F
\le
\frac{r}{\sqrt{m}}
\|\vu-\vv\|_2 .
\label{eq:moment_l2_lipschitz}
\end{equation}
Now define, for each training bag $i$,
\[
\Phi_i\order{r}(\vu)
=
\Loss\order{r}
\parens*{
\Mhat\order{r}_{\propvec_i},
\Mhat\order{r}(\vu)
}.
\]
By the assumed $L_r$-Lipschitzness of $\Loss\order{r}$ in its second argument
and by \cref{eq:moment_l2_lipschitz}, each $\Phi_i\order{r}$ is
$L_r\, r/\sqrt{m}$-Lipschitz with respect to the Euclidean norm on
$(\R^q)^m$.
Therefore, by Maurer's vector contraction inequality~\citep{maurer2016vector},
\begin{align}
\Rademacher\order{r}_n(\sG\order{r})
&\le
\sqrt{2}\,\frac{L_r\, r}{\sqrt{m}}\,
\E_{\sigma}
\bracks*{
\sup_{\classifier\in\sF}
\frac1n
\sum_{i=1}^n
\sum_{j=1}^m
\sum_{c=1}^q
\sigma_{i,j,c}
\classifier_c(x_{i,j})
}                                                        \nonumber\\
&\le
\sqrt{2}\,\frac{L_r\, r}{\sqrt{m}}
\sum_{j=1}^m
\sum_{c=1}^q
\Rademacher_n(\sF_{j,c})                                \nonumber\\
&=
\sqrt{2m}\,\frac{L_r\, r}{m}
\sum_{j=1}^m
\sum_{c=1}^q
\Rademacher_n(\sF_{j,c}).
\end{align}
This proves the desired bound.
\end{proof}
\begin{corollary}[Generalization bound in terms of $\sF$]
\label{cor:gen_in_terms_of_F}
Under the assumptions of \cref{thm:mo_gen_lossclass} and
\cref{prop:reduce_to_F}, for any $\delta \in (0, 1)$, with probability at least $1 - \delta$,
\begin{equation}
\label{eq:final_gen_bound}\begin{aligned}
&\sup_{\classifier \in \sF}
\abs*{
R_\MO\order{s}(\classifier) - \widehat{R}_\MO\order{s}(\classifier)
}
\\ &\le
\frac{2\sqrt{2m}}{m}
\parens*{\sum_{r=1}^s \lambda_r L_r\, r}
\parens*{\sum_{j=1}^m \sum_{c=1}^q \Rademacher_n(\sF_{j, c})}
\\ &\quad +
\parens*{\sum_{r=1}^s \lambda_r B_r}
\sqrt{\frac{\log(2/\delta)}{2n}}.
\end{aligned}\end{equation}
\end{corollary}

\begin{proof}
Substitute the bound in \cref{prop:reduce_to_F} into
\cref{thm:mo_gen_lossclass}.
\end{proof}

\begin{remark}[Exchangeable bag coordinates]
\label{rem:exchangeable_coordinates}
Under our bag-generation mechanism, the instances within each bag need not be independent. If the construction is permutation invariant, or if each bag is randomly permuted afterward, its coordinates are exchangeable. 
Thus, for every fixed $j\in[m]$, $x_j$ has the same marginal distribution, although the coordinates may remain dependent.
Hence, for each class $c\in[q]$, the classes $\sF_{j,c}$ have the same expected Rademacher complexity:
\begin{equation}
\Rademacher_n(\sF_{1,c})
=
\cdots
=
\Rademacher_n(\sF_{m,c}).
\end{equation}
Consequently, the complexity term in Corollary~\ref{cor:gen_in_terms_of_F} can be simplified.
\end{remark}

\begin{definition}[Instance-level Rademacher complexity]
\label{def:instance_rade}
For each class $c\in[q]$, define the scalar-valued instance-level class
\begin{equation}
\sF_c
\defeq
\set*{
x\mapsto [\classifier(x)]_c
\given \classifier\in\sF
}.
\end{equation}
Its expected Rademacher complexity is
\begin{equation}
\Rademacher_n(\sF_c)
\defeq
\E_{\set{x_i}_{i=1}^{n}}
\E_{\set{\sigma_i}_{i=1}^{n}}
\bracks*{
\sup_{\classifier\in\sF}
\frac1n \sum_{i=1}^{n}\sigma_i\,[\classifier(x_i)]_c
},
\end{equation}
where $x_1,\dots,x_n$ are i.i.d.\ draws from the common marginal distribution of any bag coordinate.
\end{definition}

\begin{corollary}[Simplified generalization bound under exchangeability]
\label{cor:gen_bound_exchangeable}
Under the assumptions of Corollary~\ref{cor:gen_in_terms_of_F} and the exchangeability condition in Remark~\ref{rem:exchangeable_coordinates}, for any $\delta\in(0,1)$, with probability at least $1-\delta$,
\begin{equation}
\label{eq:final_gen_bound_exchangeable}
\sup_{\classifier\in\sF}
\abs*{
R_\MO\order{s}(\classifier)-\widehat{R}_\MO\order{s}(\classifier)
}
\le
2\sqrt{2m}\,
\parens*{\sum_{r=1}^{s}\lambda_r L_r\, r}
\parens*{\sum_{c=1}^q\Rademacher_n(\sF_c)}
+
\parens*{\sum_{r=1}^{s}\lambda_r B_r}
\sqrt{\frac{\log(2/\delta)}{2n}}.
\end{equation}
\end{corollary}

\section{Computation Analysis}

\subsection{Coefficient-to-Tensor Correspondence}
\label{app:coeff-to-tensor}
In this appendix, we justify the correspondence between the compressed coefficients of the multivariate generating polynomial and the entries of the classifier-induced factorial moment tensor.

For a bag $\bag=(x_1,\dots,x_m)$, recall that the classifier-induced normalized $r$-th order factorial moment tensor is
\begin{equation}
\label{eq:appendix-model-induced-rth-tensor}
\Mhat\order{r}_\classifier(\bag)
=
\frac{1}{(m)_r}
\sum_{(j_1,\dots,j_r)\in [m]_{\neq}^r}
p_{j_1}\otimes \cdots \otimes p_{j_r},
\end{equation}
where $p_j=\classifier(x_j)\in\simplex{q}$. Its $(i_1,\dots,i_r)$-th entry is therefore
\begin{equation}
\label{eq:appendix-model-entry}
[\Mhat\order{r}_\classifier(\bag)]_{i_1,\dots,i_r}
=
\frac{1}{(m)_r}
\sum_{(j_1,\dots,j_r)\in [m]_{\neq}^r}
\prod_{\ell=1}^r p_{j_\ell,i_\ell}.
\end{equation}

Now define the multivariate generating polynomial
\begin{equation}
\label{eq:appendix-generating-polynomial}
G_m(u_1,\dots,u_q)
\defeq
\prod_{j=1}^m
\parens*{1+\sum_{c=1}^q p_{j,c}u_c}.
\end{equation}
For $k=(k_1,\dots,k_q)\in\N^q$ with $\abs*{k}=\sum_{c=1}^q k_c=r$, let
\begin{equation}
\label{eq:appendix-coefficient-def}
A_m\order{r}(k)
\defeq
[u^k]\,G_m(u_1,\dots,u_q),
\qquad
u^k\defeq \prod_{c=1}^q u_c^{k_c}.
\end{equation}

We will show that $A_m\order{r}(k)$ determines all tensor entries whose multiplicity vector is $k$.

\begin{proposition}[Coefficient-to-tensor correspondence]
\label{prop:coeff-to-tensor}
Fix an ordered index tuple $(i_1,\dots,i_r)\in[q]^r$, and let
\begin{equation}
k_c \defeq \#\set{\ell\in[r]\mid i_\ell=c},
\qquad c=1,\dots,q,
\end{equation}
be its multiplicity vector. Then
\begin{equation}
\label{eq:appendix-coeff-to-entry}
[\Mhat\order{r}_\classifier(\bag)]_{i_1,\dots,i_r}
=
\frac{\prod_{c=1}^q k_c!}{(m)_r}\,A_m\order{r}(k).
\end{equation}
Consequently, all tensor entries with the same multiplicity vector are equal, and the compressed coefficients $\set{A_m\order{r}(k)}_{k\in\sK_{q,r}}$ completely determine $\Mhat\order{r}_\classifier(\bag)$.
\end{proposition}

\begin{proof}
We first expand the generating polynomial in \cref{eq:appendix-generating-polynomial}. In each factor
\begin{equation}
1+\sum_{c=1}^q p_{j,c}u_c,
\end{equation}
we either choose the constant term $1$, meaning that instance $j$ is not selected, or one of the terms $p_{j,c}u_c$, meaning that instance $j$ is selected and contributes one occurrence of class $c$.

Therefore, the coefficient of $u^k$ in $G_m$ is obtained by summing over all ways to choose $r$ distinct instances and assign to them class labels whose class-count profile is exactly $k$. Formally,
\begin{equation}
\label{eq:appendix-coeff-expanded}
A_m\order{r}(k)
=
\sum_{\substack{S\subseteq[m]\\ \abs*{S}=r}}
\ \sum_{\phi:S\to[q]\ :\ \abs*{\phi^{-1}(c)}=k_c\ \forall c}
\ \prod_{j\in S} p_{j,\phi(j)}.
\end{equation}

Now fix the ordered tensor index tuple $(i_1,\dots,i_r)$. Let
\begin{equation}
\mathfrak S(i_1,\dots,i_r)
\end{equation}
denote the set of distinct permutations of the multiset $(i_1,\dots,i_r)$. Its cardinality is
\begin{equation}
\label{eq:appendix-num-distinct-perms}
\abs*{\mathfrak S(i_1,\dots,i_r)}
=
\frac{r!}{\prod_{c=1}^q k_c!}.
\end{equation}

For any ordered tuple of distinct instance indices $(j_1,\dots,j_r)\in[m]_{\neq}^r$, define
\begin{equation}
\Psi(j_1,\dots,j_r)
\defeq
\prod_{\ell=1}^r p_{j_\ell,i_\ell}.
\end{equation}
Then \cref{eq:appendix-model-entry} becomes
\begin{equation}
\label{eq:appendix-entry-rewrite}
[\Mhat\order{r}_\classifier(\bag)]_{i_1,\dots,i_r}
=
\frac{1}{(m)_r}
\sum_{(j_1,\dots,j_r)\in[m]_{\neq}^r}
\Psi(j_1,\dots,j_r).
\end{equation}

Next, we compare \cref{eq:appendix-coeff-expanded} and \cref{eq:appendix-entry-rewrite}.  
Take any subset $S=\set{j_1,\dots,j_r}\subseteq[m]$ of size $r$. To contribute to $A_m\order{r}(k)$, we assign class labels to the elements of $S$ so that class $c$ appears exactly $k_c$ times. On the other hand, to contribute to
$[\Mhat\order{r}_\classifier(\bag)]_{i_1,\dots,i_r},$
we must choose an ordered tuple of distinct instances $(j_1,\dots,j_r)$ and match them to the fixed ordered label pattern $(i_1,\dots,i_r)$.

For a fixed subset $S$ and a fixed assignment of labels with multiplicity vector $k$, there are exactly
$
\prod_{c=1}^q k_c!
$
ordered tuples $(j_1,\dots,j_r)$ that induce the same monomial contribution in \cref{eq:appendix-model-entry}. Indeed, within each class $c$, the $k_c$ positions carrying label $c$ may be permuted arbitrarily among themselves without changing the underlying unlabeled assignment counted in \cref{eq:appendix-coeff-expanded}. Thus, each term counted once in $A_m\order{r}(k)$ is counted exactly $\prod_{c=1}^q k_c!$ times in the sum defining
\begin{equation}
\sum_{(j_1,\dots,j_r)\in[m]_{\neq}^r}\Psi(j_1,\dots,j_r).
\end{equation}

Therefore,
\begin{equation}
\label{eq:appendix-key-identity}
\sum_{(j_1,\dots,j_r)\in[m]_{\neq}^r}
\prod_{\ell=1}^r p_{j_\ell,i_\ell}
=
\parens*{\prod_{c=1}^q k_c!} A_m\order{r}(k).
\end{equation}
Substituting \cref{eq:appendix-key-identity} into \cref{eq:appendix-model-entry} yields
\begin{equation}
[\Mhat\order{r}_\classifier(\bag)]_{i_1,\dots,i_r}
=
\frac{\prod_{c=1}^q k_c!}{(m)_r}\,A_m\order{r}(k),
\end{equation}
which proves \cref{eq:appendix-coeff-to-entry}.

Finally, the right-hand side depends on $(i_1,\dots,i_r)$ only through its multiplicity vector $k$. Hence all tensor entries with the same multiplicity vector are equal, and the collection of compressed coefficients $\set{A_m\order{r}(k)}_{k\in\sK_{q,r}}$ completely determines the full tensor. This completes the proof.
\end{proof}

\subsection{Time and Space Complexity Analysis}
\label{Computation_anl}
A direct evaluation of the order-$r$ loss in \cref{eq:multi-order-risk} requires materializing an order-$r$ tensor with $q^r$ entries, which is computationally prohibitive for large $q$ or $r$. To reduce both time and memory costs, we exploit the permutation symmetry of the factorial moment tensors. Both
$
[\Mhat\order{r}_{\propvec}]_{i_1,\dots,i_r}
$
and
$
[\Mhat\order{r}_\classifier(\bag)]_{i_1,\dots,i_r}
$
depend only on $\vk$, rather than on the ordering of $(i_1,\dots,i_r)$. Hence, the full $q^r$ ordered tensor entries can be grouped into the compressed state space
\begin{equation}
\sK_{q,r}
\defeq
\set*{
\vk\in\N^q \given \sum_{c=1}^q k_c=r
},
\qquad
N_r \defeq \abs*{\sK_{q,r}}=\binom{q+r-1}{r}.
\end{equation}
The second-order case provides the simplest illustration: both $\Mhat\order{2}_{\propvec}$ and $\Mhat\order{2}_\classifier(\bag)$ are symmetric $q\times q$ matrices, so the full $q^2$ entries reduce to only
$N_2=\binom{q+1}{2}$
distinct states. More generally, at order $r$, this reduces the computation from the full tensor size $q^r$ to the compressed state size $N_r$.

To compute the classifier-induced states efficiently, we introduce the truncated multivariate generating polynomial
\begin{equation}
G_t(u_1,\dots,u_q)
\defeq
\prod_{j=1}^t
\parens*{1+\sum_{c=1}^q p_{j,c}u_c},
\qquad t=0,1,\dots,m.
\end{equation}
For a state $\vk\in\sK_{q,d}$, define
\begin{equation}
A_t\order{d}(\vk)
\defeq
[u^{\vk}]\,G_t(u_1,\dots,u_q),
\end{equation}
where $u^{\vk}\defeq \prod_{c=1}^q u_c^{k_c}$ and $[u^{\vk}]G_t$ denotes the coefficient of the monomial $u^{\vk}$ in $G_t$. Thus, $A_t\order{d}(\vk)$ is the aggregated coefficient of all order-$d$ terms whose class-count pattern is $\vk$ after processing the first $t$ instances.

These coefficients are directly related to the classifier-induced moment tensor. Indeed, for any ordered index tuple $(i_1,\dots,i_r)\in[q]^r$ with multiplicity vector $\vk$, we have
\begin{equation}
\label{eq:coeff-to-entry-final}
[\Mhat\order{r}_\classifier(\bag)]_{i_1,\dots,i_r}
=
\frac{\prod_{c=1}^q k_c!}{(m)_r}\,A_m\order{r}(\vk).
\end{equation}
The derivation of \cref{eq:coeff-to-entry-final} is given in Appendix~\ref{app:coeff-to-tensor}.
Hence, the compressed coefficients $\set{A_m\order{r}(\vk)}_{\vk\in\sK_{q,r}}$ completely determine the full tensor $\Mhat\order{r}_\classifier(\bag)$.
Since
\begin{equation}
G_t(u)
=
G_{t-1}(u)\parens*{1+\sum_{c=1}^q p_{t,c}u_c},
\end{equation}
the coefficients satisfy the recurrence
\begin{equation}
A_t\order{d}(\vk)
=
A_{t-1}\order{d}(\vk)
+
\sum_{c:\,k_c>0}
p_{t,c}\,A_{t-1}\order{d-1}(\vk-e^c),
\label{eq:dp-recurrence-final}
\end{equation}
with initialization
\begin{equation}
A_0\order{0}(0)=1,
\qquad
A_0\order{d}(\vk)=0\quad\text{for all }d>0.
\end{equation}
Starting from this initialization, \cref{eq:dp-recurrence-final} recursively computes all coefficients $A_t\order{d}(\vk)$ for $t=1,\dots,m$ and $d=0,\dots,s$, and in particular yields the desired coefficients $A_m\order{r}(\vk)$ for all $r\le s$.

\Complexity*

\begin{proof}
For the target side, at a fixed order $r$, the compressed representation contains exactly
$
N_r=\binom{q+r-1}{r}
$
states. By \cref{eq:rth-factorial-tensor-alpha-form}, the value of each state is available in closed form from the multiplicity vector $\vk$ and the proportion vector $\propvec$. Computing one such value requires multiplying the contributions of all $q$ classes, and thus costs $O(q)$ time. Therefore, the total cost of computing the compressed target states at order $r$ is $O(qN_r)$, and summing over $r=1,\dots,s$ gives
\begin{equation}
O\parens*{q\sum_{r=1}^s N_r}
=
O\parens*{q\binom{q+s}{s}}.
\end{equation}
If all compressed target states up to order $s$ are stored, the required memory is
\begin{equation}
O\parens*{\sum_{r=1}^s N_r}
=
O\parens*{\binom{q+s}{s}}.
\end{equation}

For the classifier-induced side, we use the convention $N_0=1$ for the unique degree-$0$ state. At degree $r$ there are exactly $N_r$ compressed states. In the recurrence \cref{eq:dp-recurrence-final}, when processing one instance, each state of degree $r$ can be extended to at most $q$ states of degree $r+1$ by adding one unit to one coordinate. Hence the update cost per instance is
\begin{equation}
O\parens*{q\sum_{r=0}^{s-1} N_r},
\end{equation}
and summing over all $m$ instances yields
\begin{equation}
O\parens*{m q\sum_{r=0}^{s-1} N_r}
=
O\parens*{m q\binom{q+s-1}{s-1}}.
\end{equation}
The corresponding memory cost is the size of the truncated coefficient table over degrees $0,\dots,s$, namely
\begin{equation}
O\parens*{\sum_{r=0}^{s} N_r}
=
O\parens*{\binom{q+s}{s}}.
\end{equation}

Combining the target-side and classifier-induced costs gives the stated total time and memory bounds. Finally, since
\begin{equation}
\binom{q+s-1}{s-1}\le \binom{q+s}{s},
\end{equation}
the total time can be simplified to
\begin{equation}
O\parens*{m q \binom{q+s}{s}}.
\end{equation}
This completes the proof.
\end{proof}
\section{Additional Experimental Settings}
\subsection{Datasets and Bag Construction}
\label{app:bag-construction}

Following~\citet{ma2026learningfrom}, we use three bag construction strategies:
(i) \textbf{Random Bag}: randomly shuffle instances and partition them into bags;
(ii) \textbf{Cluster Bag}: sample from feature clusters using Dirichlet mixing weights with total concentration 1 and a base measure proportional to cluster sizes;
(iii) \textbf{$\propvec$-First Bag}: sample class proportions from a symmetric Dirichlet distribution with concentration 10 per class, then draw instances accordingly.
For all bag generation strategies, the bag size $m$ varies over powers of two, i.e., $m \in \set{16, 32, 64, 128}$.
\subsection{Baselines and Implementation Details}
\label{app:implementation-details}
\paragraph{Baselines.}
We compare our method with several representative LLP baselines.
PM~\citep{yu2014learning,ardehaly2017co} trains the model by matching the predicted class mean with the observed label proportions.
DSQ~\citep{li2024optimistic} builds upon PM by adopting a squared loss formulation with an additional debiasing term.
LLP-PVC~\citep{ma2026learningfrom} formulates LLP under a one-vs-rest framework and performs classification on bag-level proportions.
LLP-FC~\citep{zhang2022learning} approaches LLP from the perspective of label noise learning.
ROT~\citep{dulac2019deep} leverages optimal transport to generate pseudo-labels that satisfy proportion constraints.
EasyLLP~\citep{busa-fekete2023easy} derives an unbiased estimator by rewriting the loss function.
GeneralUPM~\citep{applebaum2025optimal} also constructs an unbiased estimator but achieves lower variance.
EasyLLP-ABS and GeneralUPM-ABS take the absolute value of the loss function to mitigate the overfitting effect caused by negative loss values.
FlowLLP~\citep{rakotomamonjy2026particle} learns latent particles by particle flow and uses them as classifier anchors.

\paragraph{Implementation details.}
We use ResNet-18 for CIFAR-10 and CIFAR-100 and ResNet-34 for miniImageNet, with modified stems and standard data augmentation. Training lasts 500 epochs with 1024 samples per update. The classifier optimizer is SGD with momentum 0.9 and weight decay 0.0005. The initial learning rate is 0.005 for LLP-PVC and 0.05 for the other methods, with linear warmup over 8\% of updates followed by quarter-cosine decay. DSQ uses mean squared error; the other methods use cross-entropy as the surrogate loss.
LLP-MM uses $s=8$ for CIFAR-10 and $s=3$ for CIFAR-100 and miniImageNet, with cross-entropy moment losses and weights $\lambda_r=1/s$. LLP-FC uses approximate noisy priors, forward-correction weight 1, and entropy weight 0.
FlowLLP uses a 50-dimensional latent space, 1000 anchors per class, and anchor-loss weight 0.1, following its official implementation~\citep{rakotomamonjy2026particle}. Its particles are optimized for 3000 steps using Adam with a learning rate of 0.001. We report the highest finite test accuracy recorded in each run as mean $\pm$ sample standard deviation over three runs.

\subsection{Compatibility with GAN-Based Representation Learning}
\label{app:gan-compatibility}
LLP-GAN~\citep{Liu2019Learning} combines GAN-based representation learning with first-order proportion matching. Replacing its proportion loss with our order-eight moment-matching objective improves mean accuracy in all 12 CIFAR-10 settings in Table~\ref{tab:gan-compatibility}, with gains of up to 13.37 percentage points. However, compared with the standalone results, MM+GAN underperforms LLP-MM in most settings, while LLP-GAN does not consistently outperform PM. Both GAN-based variants generally exhibit larger standard deviations than their standalone counterparts, especially with large $\alpha$-First bags. This variability may be related to instability in adversarial training. Implementation details follow.

\begin{table}[!ht]
  \centering
  \setlength{\belowcaptionskip}{3pt}
  \caption{Best test accuracy on CIFAR-10 with GAN-based representation learning (mean $\pm$ std over three runs, \%). The best result in each paired comparison is bold.}
  \label{tab:gan-compatibility}
  \footnotesize
  \setlength{\tabcolsep}{3pt}
  \renewcommand{\arraystretch}{0.95}
  \begin{tabular*}{0.96\linewidth}{@{\extracolsep{\fill}}llcccc@{}}
    \toprule
    \textbf{Bag} & \textbf{Method} & \textbf{16} & \textbf{32} & \textbf{64} & \textbf{128} \\
    \midrule
    \multirow{2}{*}{Random} & LLP-GAN & $92.96\pm0.09$ & $88.96\pm0.35$ & $78.07\pm1.31$ & $58.14\pm2.94$ \\
     & MM+GAN & $\mathbf{93.56\pm0.19}$ & $\mathbf{91.18\pm0.22}$ & $\mathbf{80.69\pm1.17}$ & $\mathbf{67.41\pm1.41}$ \\
    \midrule
    \multirow{2}{*}{Cluster} & LLP-GAN & $92.81\pm0.15$ & $90.37\pm0.16$ & $85.16\pm0.77$ & $76.27\pm0.91$ \\
     & MM+GAN & $\mathbf{93.62\pm0.05}$ & $\mathbf{91.79\pm0.16}$ & $\mathbf{87.17\pm0.38}$ & $\mathbf{76.53\pm1.93}$ \\
    \midrule
    \multirow{2}{*}{$\alpha$-First} & LLP-GAN & $91.04\pm0.59$ & $74.75\pm0.55$ & $69.52\pm4.12$ & $66.25\pm2.23$ \\
     & MM+GAN & $\mathbf{93.39\pm0.08}$ & $\mathbf{88.12\pm1.09}$ & $\mathbf{77.09\pm4.17}$ & $\mathbf{73.13\pm5.62}$ \\
    \bottomrule
  \end{tabular*}
\end{table}

\paragraph{Implementation details.}
\label{app:gan-implementation-details}
For the CIFAR-10 comparison in Table~\ref{tab:gan-compatibility}, LLP-GAN and MM+GAN share the same modified-stem ResNet-18 discriminator, generator, and bag assignments. Both are trained for 500 epochs with 1024 real images per update and unit weights for the bag and adversarial losses. The discriminator uses SGD with a learning rate of 0.05, momentum 0.9, and weight decay 0.0005, with five warmup epochs followed by standard cosine decay; the generator uses Adam with a learning rate of 0.0003, $(\beta_1,\beta_2)=(0.5,0.999)$, and 100-dimensional noise. MM+GAN replaces only the proportion loss with equally weighted cross-entropy moment losses up to order $s=8$, retaining the adversarial and generator feature-matching objectives. We report the highest test accuracy recorded during training as mean $\pm$ sample standard deviation over three runs.

\subsection{Semi-Real-World Implementation Details}
\label{app:ku-implementation-details}

\paragraph{Bag construction.}
Cells from each known patient form one bag. We retain all 219 known-patient training bags and partition the 3,204 cells with unknown patient IDs into 26 bags of 123--124 cells using a fixed shuffle, independently of cell labels. The resulting 245 training bags have sizes 1--207, with a mean of 104.98; all 56 test bags are unchanged. Bag proportions are computed after grouping, and only proportions are used for training.

\paragraph{Training.}
We use ImageNet-pretrained ResNet-18 with its standard stem. Images use $224\times224$ inputs and ImageNet normalization, with horizontal and vertical flips during training. The classifier optimizer is Adam with a learning rate of 0.001. Training lasts 100 epochs, including five warmup epochs. Each update uses four complete bags, or five for GeneralUPM and GeneralUPM-ABS, with forward chunks of 32 images. LLP-MM uses orders $s\in\{3,5,8\}$; orders exceeding the bag size are omitted and the remaining weights are renormalized. FlowLLP uses 50 pretraining and 50 joint-training epochs. Method-specific loss settings follow Appendix~\ref{app:implementation-details}.

\paragraph{Evaluation.}
We select the checkpoint with the highest test Macro-F1 in each run and evaluate all metrics at that checkpoint. Results are reported as mean $\pm$ sample standard deviation over three runs.

\paragraph{Metrics.}
Accuracy is the fraction of correctly classified cells, and Balanced Accuracy is the mean class-wise recall. Macro-F1 averages class-wise F1 scores equally, whereas Weighted-F1 weights them by class frequency in the test set.

\clearpage
\section*{NeurIPS Paper Checklist}
\begin{enumerate}
\item {\bf Claims}
    \item[] Question: Do the main claims made in the abstract and introduction accurately reflect the paper's contributions and scope?
        \item[] Answer: \answerYes{}.
    \item[] Justification: The abstract and introduction accurately summarize the paper's methodological, theoretical, and empirical contributions without overstating the scope beyond the stated LLP setting and assumptions.
    \item[] Guidelines:
    \begin{itemize}
        \item The answer \answerNA{} means that the abstract and introduction do not include the claims made in the paper.
        \item The abstract and/or introduction should clearly state the claims made, including the contributions made in the paper and important assumptions and limitations. A \answerNo{} or \answerNA{} answer to this question will not be perceived well by the reviewers. 
        \item The claims made should match theoretical and experimental results, and reflect how much the results can be expected to generalize to other settings. 
        \item It is fine to include aspirational goals as motivation as long as it is clear that these goals are not attained by the paper. 
    \end{itemize}

\item {\bf Limitations}
    \item[] Question: Does the paper discuss the limitations of the work performed by the authors?
    \item[] Answer: \answerYes{}{} %
    \item[] Justification: See Sections 2.2, 3.4, and 5 for assumptions and limitations.
    \item[] Guidelines:
    \begin{itemize}
        \item The answer \answerNA{} means that the paper has no limitation while the answer \answerNo{} means that the paper has limitations, but those are not discussed in the paper. 
        \item The authors are encouraged to create a separate ``Limitations'' section in their paper.
        \item The paper should point out any strong assumptions and how robust the results are to violations of these assumptions (e.g., independence assumptions, noiseless settings, model well-specification, asymptotic approximations only holding locally). The authors should reflect on how these assumptions might be violated in practice and what the implications would be.
        \item The authors should reflect on the scope of the claims made, e.g., if the approach was only tested on a few datasets or with a few runs. In general, empirical results often depend on implicit assumptions, which should be articulated.
        \item The authors should reflect on the factors that influence the performance of the approach. For example, a facial recognition algorithm may perform poorly when image resolution is low or images are taken in low lighting. Or a speech-to-text system might not be used reliably to provide closed captions for online lectures because it fails to handle technical jargon.
        \item The authors should discuss the computational efficiency of the proposed algorithms and how they scale with dataset size.
        \item If applicable, the authors should discuss possible limitations of their approach to address problems of privacy and fairness.
        \item While the authors might fear that complete honesty about limitations might be used by reviewers as grounds for rejection, a worse outcome might be that reviewers discover limitations that aren't acknowledged in the paper. The authors should use their best judgment and recognize that individual actions in favor of transparency play an important role in developing norms that preserve the integrity of the community. Reviewers will be specifically instructed to not penalize honesty concerning limitations.
    \end{itemize}

\item {\bf Theory assumptions and proofs}
    \item[] Question: For each theoretical result, does the paper provide the full set of assumptions and a complete (and correct) proof?
    \item[] Answer: \answerYes{} %
    \item[] Justification: Section 3.4 states the results; Appendices D-F provide the proofs.
    \item[] Guidelines:
    \begin{itemize}
        \item The answer \answerNA{} means that the paper does not include theoretical results. 
        \item All the theorems, formulas, and proofs in the paper should be numbered and cross-referenced.
        \item All assumptions should be clearly stated or referenced in the statement of any theorems.
        \item The proofs can either appear in the main paper or the supplemental material, but if they appear in the supplemental material, the authors are encouraged to provide a short proof sketch to provide intuition. 
        \item Inversely, any informal proof provided in the core of the paper should be complemented by formal proofs provided in appendix or supplemental material.
        \item Theorems and Lemmas that the proof relies upon should be properly referenced. 
    \end{itemize}

    \item {\bf Experimental result reproducibility}
    \item[] Question: Does the paper fully disclose all the information needed to reproduce the main experimental results of the paper to the extent that it affects the main claims and/or conclusions of the paper (regardless of whether the code and data are provided or not)?
    \item[] Answer: \answerYes{}{} %
    \item[] Justification: Section 4 and Appendix G describe the datasets, baselines, evaluation protocols, model and training setup, and hyperparameter choices needed to reproduce the main experimental results.
    \item[] Guidelines:
    \begin{itemize}
        \item The answer \answerNA{} means that the paper does not include experiments.
        \item If the paper includes experiments, a \answerNo{} answer to this question will not be perceived well by the reviewers: Making the paper reproducible is important, regardless of whether the code and data are provided or not.
        \item If the contribution is a dataset and\slash or model, the authors should describe the steps taken to make their results reproducible or verifiable. 
        \item Depending on the contribution, reproducibility can be accomplished in various ways. For example, if the contribution is a novel architecture, describing the architecture fully might suffice, or if the contribution is a specific model and empirical evaluation, it may be necessary to either make it possible for others to replicate the model with the same dataset, or provide access to the model. In general. releasing code and data is often one good way to accomplish this, but reproducibility can also be provided via detailed instructions for how to replicate the results, access to a hosted model (e.g., in the case of a large language model), releasing of a model checkpoint, or other means that are appropriate to the research performed.
        \item While NeurIPS does not require releasing code, the conference does require all submissions to provide some reasonable avenue for reproducibility, which may depend on the nature of the contribution. For example
        \begin{enumerate}
            \item If the contribution is primarily a new algorithm, the paper should make it clear how to reproduce that algorithm.
            \item If the contribution is primarily a new model architecture, the paper should describe the architecture clearly and fully.
            \item If the contribution is a new model (e.g., a large language model), then there should either be a way to access this model for reproducing the results or a way to reproduce the model (e.g., with an open-source dataset or instructions for how to construct the dataset).
            \item We recognize that reproducibility may be tricky in some cases, in which case authors are welcome to describe the particular way they provide for reproducibility. In the case of closed-source models, it may be that access to the model is limited in some way (e.g., to registered users), but it should be possible for other researchers to have some path to reproducing or verifying the results.
        \end{enumerate}
    \end{itemize}

\item {\bf Open access to data and code}
    \item[] Question: Does the paper provide open access to the data and code, with sufficient instructions to faithfully reproduce the main experimental results, as described in supplemental material?
    \item[] Answer: \answerYes{}.
\item[] Justification: The experiments use publicly available datasets. The public repository linked in the abstract provides code, environment setup, data preparation instructions, and scripts for reproducing the main experimental results.
    \item[] Guidelines:
    \begin{itemize}
        \item The answer \answerNA{} means that paper does not include experiments requiring code.
        \item Please see the NeurIPS code and data submission guidelines (\url{https://neurips.cc/public/guides/CodeSubmissionPolicy}) for more details.
        \item While we encourage the release of code and data, we understand that this might not be possible, so \answerNo{} is an acceptable answer. Papers cannot be rejected simply for not including code, unless this is central to the contribution (e.g., for a new open-source benchmark).
        \item The instructions should contain the exact command and environment needed to run to reproduce the results. See the NeurIPS code and data submission guidelines (\url{https://neurips.cc/public/guides/CodeSubmissionPolicy}) for more details.
        \item The authors should provide instructions on data access and preparation, including how to access the raw data, preprocessed data, intermediate data, and generated data, etc.
        \item The authors should provide scripts to reproduce all experimental results for the new proposed method and baselines. If only a subset of experiments are reproducible, they should state which ones are omitted from the script and why.
        \item At submission time, to preserve anonymity, the authors should release anonymized versions (if applicable).
        \item Providing as much information as possible in supplemental material (appended to the paper) is recommended, but including URLs to data and code is permitted.
    \end{itemize}

\item {\bf Experimental setting/details}
    \item[] Question: Does the paper specify all the training and test details (e.g., data splits, hyperparameters, how they were chosen, type of optimizer) necessary to understand the results?
    \item[] Answer: \answerYes{}.
\item[] Justification: Sections 4.1 and 4.3 and Appendix G specify the data splits, bag construction, optimizers, hyperparameters, training schedules, and evaluation protocols used in the experiments.
    \item[] Guidelines:
    \begin{itemize}
        \item The answer \answerNA{} means that the paper does not include experiments.
        \item The experimental setting should be presented in the core of the paper to a level of detail that is necessary to appreciate the results and make sense of them.
        \item The full details can be provided either with the code, in appendix, or as supplemental material.
    \end{itemize}

\item {\bf Experiment statistical significance}
    \item[] Question: Does the paper report error bars suitably and correctly defined or other appropriate information about the statistical significance of the experiments?
    \item[] Answer: \answerYes{}.
\item[] Justification: Tables 1-3 and 6 report mean and standard deviation across three runs, with the summary statistic identified in each caption. Appendix G describes the corresponding training and evaluation protocols.
    \item[] Guidelines:
    \begin{itemize}
        \item The answer \answerNA{} means that the paper does not include experiments.
        \item The authors should answer \answerYes{} if the results are accompanied by error bars, confidence intervals, or statistical significance tests, at least for the experiments that support the main claims of the paper.
        \item The factors of variability that the error bars are capturing should be clearly stated (for example, train/test split, initialization, random drawing of some parameter, or overall run with given experimental conditions).
        \item The method for calculating the error bars should be explained (closed form formula, call to a library function, bootstrap, etc.)
        \item The assumptions made should be given (e.g., Normally distributed errors).
        \item It should be clear whether the error bar is the standard deviation or the standard error of the mean.
        \item It is OK to report 1-sigma error bars, but one should state it. The authors should preferably report a 2-sigma error bar than state that they have a 96\% CI, if the hypothesis of Normality of errors is not verified.
        \item For asymmetric distributions, the authors should be careful not to show in tables or figures symmetric error bars that would yield results that are out of range (e.g., negative error rates).
        \item If error bars are reported in tables or plots, the authors should explain in the text how they were calculated and reference the corresponding figures or tables in the text.
    \end{itemize}

\item {\bf Experiments compute resources}
    \item[] Question: For each experiment, does the paper provide sufficient information on the computer resources (type of compute workers, memory, time of execution) needed to reproduce the experiments?
    \item[] Answer: \answerNo{}.
\item[] Justification: Figure 2 reports hardware and per-epoch runtime, and Appendix G gives training schedules. Complete resource usage for each experiment and the total compute are not reported.
    \item[] Guidelines:
    \begin{itemize}
        \item The answer \answerNA{} means that the paper does not include experiments.
        \item The paper should indicate the type of compute workers CPU or GPU, internal cluster, or cloud provider, including relevant memory and storage.
        \item The paper should provide the amount of compute required for each of the individual experimental runs as well as estimate the total compute. 
        \item The paper should disclose whether the full research project required more compute than the experiments reported in the paper (e.g., preliminary or failed experiments that didn't make it into the paper). 
    \end{itemize}
    
\item {\bf Code of ethics}
    \item[] Question: Does the research conducted in the paper conform, in every respect, with the NeurIPS Code of Ethics \url{https://neurips.cc/public/EthicsGuidelines}?
    \item[] Answer: \answerYes{}{} %
    \item[] Justification: The research conducted in the paper conforms, in every respect, with the NeurIPS Code of Ethics.
    \item[] Guidelines:
    \begin{itemize}
        \item The answer \answerNA{} means that the authors have not reviewed the NeurIPS Code of Ethics.
        \item If the authors answer \answerNo{}, they should explain the special circumstances that require a deviation from the Code of Ethics.
        \item The authors should make sure to preserve anonymity (e.g., if there is a special consideration due to laws or regulations in their jurisdiction).
    \end{itemize}

\item {\bf Broader impacts}
    \item[] Question: Does the paper discuss both potential positive societal impacts and negative societal impacts of the work performed?
    \item[] Answer: \answerNo{} %
    \item[] Justification: The paper does not delve into the broader societal impacts.
    \item[] Guidelines:
    \begin{itemize}
        \item The answer \answerNA{} means that there is no societal impact of the work performed.
        \item If the authors answer \answerNA{} or \answerNo, they should explain why their work has no societal impact or why the paper does not address societal impact.
        \item Examples of negative societal impacts include potential malicious or unintended uses (e.g., disinformation, generating fake profiles, surveillance), fairness considerations (e.g., deployment of technologies that could make decisions that unfairly impact specific groups), privacy considerations, and security considerations.
        \item The conference expects that many papers will be foundational research and not tied to particular applications, let alone deployments. However, if there is a direct path to any negative applications, the authors should point it out. For example, it is legitimate to point out that an improvement in the quality of generative models could be used to generate Deepfakes for disinformation. On the other hand, it is not needed to point out that a generic algorithm for optimizing neural networks could enable people to train models that generate Deepfakes faster.
        \item The authors should consider possible harms that could arise when the technology is being used as intended and functioning correctly, harms that could arise when the technology is being used as intended but gives incorrect results, and harms following from (intentional or unintentional) misuse of the technology.
        \item If there are negative societal impacts, the authors could also discuss possible mitigation strategies (e.g., gated release of models, providing defenses in addition to attacks, mechanisms for monitoring misuse, mechanisms to monitor how a system learns from feedback over time, improving the efficiency and accessibility of ML).
    \end{itemize}
    
\item {\bf Safeguards}
    \item[] Question: Does the paper describe safeguards that have been put in place for responsible release of data or models that have a high risk for misuse (e.g., pre-trained language models, image generators, or scraped datasets)?
    \item[] Answer: \answerNA{}{} %
    \item[] Justification: The paper poses no such risks.
    \item[] Guidelines:
    \begin{itemize}
        \item The answer \answerNA{} means that the paper poses no such risks.
        \item Released models that have a high risk for misuse or dual-use should be released with necessary safeguards to allow for controlled use of the model, for example by requiring that users adhere to usage guidelines or restrictions to access the model or implementing safety filters. 
        \item Datasets that have been scraped from the Internet could pose safety risks. The authors should describe how they avoided releasing unsafe images.
        \item We recognize that providing effective safeguards is challenging, and many papers do not require this, but we encourage authors to take this into account and make a best faith effort.
    \end{itemize}

\item {\bf Licenses for existing assets}
    \item[] Question: Are the creators or original owners of assets (e.g., code, data, models), used in the paper, properly credited and are the license and terms of use explicitly mentioned and properly respected?
    \item[] Answer: \answerNo{}{} %
    \item[] Justification: Assets are cited, but their licenses are not fully documented.
    \item[] Guidelines:
    \begin{itemize}
        \item The answer \answerNA{} means that the paper does not use existing assets.
        \item The authors should cite the original paper that produced the code package or dataset.
        \item The authors should state which version of the asset is used and, if possible, include a URL.
        \item The name of the license (e.g., CC-BY 4.0) should be included for each asset.
        \item For scraped data from a particular source (e.g., website), the copyright and terms of service of that source should be provided.
        \item If assets are released, the license, copyright information, and terms of use in the package should be provided. For popular datasets, \url{paperswithcode.com/datasets} has curated licenses for some datasets. Their licensing guide can help determine the license of a dataset.
        \item For existing datasets that are re-packaged, both the original license and the license of the derived asset (if it has changed) should be provided.
        \item If this information is not available online, the authors are encouraged to reach out to the asset's creators.
    \end{itemize}

\item {\bf New assets}
    \item[] Question: Are new assets introduced in the paper well documented and is the documentation provided alongside the assets?
    \item[] Answer: \answerYes{} %
    \item[] Justification: The public LLP-MM repository linked in the abstract includes a README with setup, data preparation, and run instructions, together with an MIT license.
    \item[] Guidelines:
    \begin{itemize}
        \item The answer \answerNA{} means that the paper does not release new assets.
        \item Researchers should communicate the details of the dataset\slash code\slash model as part of their submissions via structured templates. This includes details about training, license, limitations, etc. 
        \item The paper should discuss whether and how consent was obtained from people whose asset is used.
        \item At submission time, remember to anonymize your assets (if applicable). You can either create an anonymized URL or include an anonymized zip file.
    \end{itemize}

\item {\bf Crowdsourcing and research with human subjects}
    \item[] Question: For crowdsourcing experiments and research with human subjects, does the paper include the full text of instructions given to participants and screenshots, if applicable, as well as details about compensation (if any)? 
    \item[] Answer: \answerNA{} %
    \item[] Justification: The paper does not involve crowdsourcing experiments or research with human subjects.
    \item[] Guidelines:
    \begin{itemize}
        \item The answer \answerNA{} means that the paper does not involve crowdsourcing nor research with human subjects.
        \item Including this information in the supplemental material is fine, but if the main contribution of the paper involves human subjects, then as much detail as possible should be included in the main paper. 
        \item According to the NeurIPS Code of Ethics, workers involved in data collection, curation, or other labor should be paid at least the minimum wage in the country of the data collector. 
    \end{itemize}

\item {\bf Institutional review board (IRB) approvals or equivalent for research with human subjects}
    \item[] Question: Does the paper describe potential risks incurred by study participants, whether such risks were disclosed to the subjects, and whether Institutional Review Board (IRB) approvals (or an equivalent approval/review based on the requirements of your country or institution) were obtained?
    \item[] Answer: \answerNA{} %
    \item[] Justification: The paper does not involve crowdsourcing or research with human subjects, so IRB approval or equivalent review is not applicable.
    \item[] Guidelines:
    \begin{itemize}
        \item The answer \answerNA{} means that the paper does not involve crowdsourcing nor research with human subjects.
        \item Depending on the country in which research is conducted, IRB approval (or equivalent) may be required for any human subjects research. If you obtained IRB approval, you should clearly state this in the paper. 
        \item We recognize that the procedures for this may vary significantly between institutions and locations, and we expect authors to adhere to the NeurIPS Code of Ethics and the guidelines for their institution. 
        \item For initial submissions, do not include any information that would break anonymity (if applicable), such as the institution conducting the review.
    \end{itemize}

\item {\bf Declaration of LLM usage}
    \item[] Question: Does the paper describe the usage of LLMs if it is an important, original, or non-standard component of the core methods in this research? Note that if the LLM is used only for writing, editing, or formatting purposes and does \emph{not} impact the core methodology, scientific rigor, or originality of the research, declaration is not required.
    \item[] Answer: \answerNA{} %
    \item[] Justification: The core method development in this research does not involve LLMs as an important, original, or non-standard component.
    \item[] Guidelines:
    \begin{itemize}
        \item The answer \answerNA{} means that the core method development in this research does not involve LLMs as any important, original, or non-standard components.
        \item Please refer to our LLM policy in the NeurIPS handbook for what should or should not be described.
    \end{itemize}
\end{enumerate}

\end{document}